\documentclass{article} % DO NOT CHANGE THIS
\usepackage{arxiv}  % DO NOT CHANGE THIS

\usepackage{times}  % DO NOT CHANGE THIS
\usepackage{helvet}  % DO NOT CHANGE THIS
\usepackage{courier}  % DO NOT CHANGE THIS
\usepackage{graphicx} % DO NOT CHANGE THIS
\usepackage{natbib}  % DO NOT CHANGE THIS AND DO NOT ADD ANY OPTIONS TO IT
\usepackage{caption} % DO NOT CHANGE THIS AND DO NOT ADD ANY OPTIONS TO IT
\usepackage{algorithm}
\usepackage{algorithmic}

\usepackage{newfloat}
\usepackage{listings}
\renewcommand{\shorttitle}{Robust Phase Retrieval}

\title{A Sample Efficient Alternating Minimization-based Algorithm For Robust Phase Retrieval}
\author{
  Adarsh Barik\\
  Institute of Data Science\\
  National University of Singapore\\
  Singapore, 117602 \\
  \texttt{abarik@nus.edu.sg} \\
  \And
  Anand Krishna\\
  Department of Electrical and Computer Engineering\\
  National University of Singapore\\
  Singapore, 117583 \\
  \texttt{akr@nus.edu.sg} \\
  \And
  Vincent Y. F. Tan\\
  Department of Mathematics\\
  Department of Electrical and Computer Engineering\\
  National University of Singapore\\
  Singapore, 119077 \\
  \texttt{vtan@nus.edu.sg} \\
}

\usepackage{bibentry}
\usepackage{mathtools}
\usepackage{amssymb}
\usepackage{subcaption}
\usepackage{bm}
\usepackage{amsthm}
\usepackage{thm-restate}
\usepackage[inline]{enumitem}
\usepackage{booktabs}
\newcommand{\btheta}{\bm{\theta}}
\newcommand{\bx}{\bm{x}}
\newcommand{\bz}{\bm{z}}
\newcommand{\calK}{\mathcal{K}}
\newcommand{\bmeta}{\bm{\eta}}
\newtheorem{lemma}{Lemma}
\newtheorem{theorem}{Theorem}
\newtheorem{corollary}{Corollary}

\newtheorem{definition}{Definition}
\newtheorem{problem}{Problem}
\newtheorem{assumption}{Assumption}

\newcommand{\inner}[2]{\langle #1, #2 \rangle}
\newcommand{\real}{\mathbb{R}}
\newcommand{\normal}{\mathcal{N}}

\newcommand{\calO}{\mathcal{O}}

\newcommand{\hU}{\hat{U}}

\newcommand{\oracle}{\textsc{LSQ-PHASE-ORACLE}}
\newcommand{\altmin}{\textsc{ALT-MIN-PHASE}}

\newcommand{\E}[1]{\mathbb{E}\left[ #1 \right]}
\newcommand{\prob}[1]{\mathbb{P}\left[ #1 \right]}

\newcommand{\given}{\,\middle\vert\,}
\newcommand{\abs}[1]{\left| #1 \right|}
\newcommand{\newO}[1]{\mathcal{O}\left(#1\right)}

\mathtoolsset{showonlyrefs}

\begin{document}

\maketitle

\begin{abstract}
In this work, we study the robust phase retrieval problem where the task is to recover an unknown signal  $\btheta^* \in \real^d$ in the presence of potentially arbitrarily corrupted magnitude-only linear measurements. We propose an alternating minimization approach that incorporates an oracle solver for a non-convex optimization problem as a subroutine. Our algorithm guarantees convergence to $\btheta^*$ and provides an explicit polynomial dependence of the convergence rate on the fraction of corrupted measurements. We then provide an efficient construction of the aforementioned oracle under a sparse arbitrary outliers model and offer valuable insights into the geometric properties of the loss landscape in phase retrieval with corrupted measurements. Our proposed oracle avoids the need for computationally intensive spectral initialization, using a simple gradient descent algorithm with a constant step size and random initialization instead. Additionally, our overall algorithm achieves nearly linear sample complexity, $\calO(d \, \mathrm{polylog}(d))$.
\end{abstract}


\section{Introduction}
\label{sec: introduction}

The problem of phase retrieval consists of recovering an unknown target signal from intensity-only measurements. It has gained wide interest in many areas of engineering, applied physics, and machine learning~\citep{dong2023phase}, such as optics~\citep{walther1963question}, X-ray crystallography~\citep{millane1990phase}, inference of DNA structure~\citep{stefik1978inferring}, and more. Mathematically, the task is to learn an unknown signal $\btheta^* \in \real^d$ from $n$ magnitude-only linear measurements\footnote{While the phase retrieval problem is also studied in the complex domain, we only focus on real signals in this work.}. To ensure the smoothness of the loss function, we describe the data generation process of the phase retrieval problem in the following quadratic form:
\begin{align}
    \label{eq: standard phase retrieval}
    y_i = \inner{\bx_i}{\btheta^*}^2, \quad i \in [n]~,
\end{align}
where $[n]$ is a shorthand to denote the set $\{1, \ldots, n\}$. Borrowing terminology from linear regression literature~\citep{bakshi2021robust}, we term $\bx_i \in \real^d$ to be the $i$-th covariate vector and $y_i \in \real$ to be the $i$-th response in Equation~\eqref{eq: standard phase retrieval}. The tuple $(\bx_i, y_i) \in \real^{d+1}$ is the $i$-th measurement. We study the problem under Gaussian design where each entry $x_{ij}$, for $i \in [n]$ and $j \in [d]$, is drawn i.i.d. from the standard normal distribution $\normal(0, 1)$. Due to the quadratic nature of these measurements, the phase information is lost, making the recovery of the true signal $\btheta^* \in \real^d$ significantly challenging. Consequently, one can only hope to recover the signal up to a variation of its sign. Therefore, the output $\hat{\btheta}$ of any algorithm is measured (in the context of real signals) by evaluating $d(\hat{\btheta}, \btheta^*) \coloneqq \min\big\{ \| \hat{\btheta} - \btheta^* \|, \| \hat{\btheta} + \btheta^* \| \big\}$.
% \begin{align}
%     \label{eq: performance measure}
%     d(\hat{\btheta}, \btheta^*) \coloneqq \min\big( \| \hat{\btheta} - \btheta^* \|, \| \hat{\btheta} + \btheta^* \| \big)~.
% \end{align}
The difficulty is further compounded when some measurements can be {\em   arbitrarily corrupted}. In this work, we aim to develop an algorithm for the phase retrieval problem that is robust to arbitrary corruption in $k$ out of $n$ measurements. To that end, we want to design and analyze algorithms to find $\hat{\btheta}$ and   understand how $d(\hat{\btheta}, \btheta^*)$ depends on $k$ and $n$. Below, we briefly outline the main contributions of this work:
\begin{itemize}
    \item We propose an alternating minimization-based algorithm for phase retrieval with corrupted measurements. With only $n = \Omega\big(\frac{d\,\mathrm{polylog}( d)}{\epsilon^2 \log(\frac{1}{\epsilon})}\big)$ quadratic measurements and a corruption proportion of  $\epsilon = \frac{k}{n}$, our algorithm achieves $d(\hat{\btheta}, \btheta^*) = \tilde{\calO}\big(\sqrt{\epsilon}\big)$ with high probability, even under a strong corruption model~\citep{bakshi2021robust}.
    %The convergence rate is slower ($r=\frac{1}{2}$) when $\hat{\btheta}$ is far from $\btheta^*$ but accelerates ($r=1$) as $\hat{\btheta}$ moves closer to $\btheta^*$.
    To the best of our knowledge, this is the first algorithm for robust phase retrieval that provides an explicit expression for $d(\hat{\btheta}, \btheta^*)$ as a function of $k$ and $n$. We also show that our algorithm stops in a finite number of iterations.
    \item Our high-probability guarantees are valid in regimes where $\epsilon \sqrt{\log \epsilon^{-1}} \log^2(\epsilon n) \to 0$
    %$\frac{k}{n} \sqrt{\log \frac{n}{k}} \log^2(k) \to 0$
    as $n \to \infty$. This includes the vanishing proportion regimes such as $k = n^{1 - p}$ for $p \in (0, 1]$.
    %The results can be extended to other corruption regimes such as constant corruption proportion, but they require $\Omega(d^2)$ measurements in those cases and a vanishing rate is unachievable in those regimes.
    \item In the first stage of our analysis, we assume the existence of an oracle capable of solving a nonconvex optimization problem to global optimality. In the second stage, we demonstrate that such an oracle can be efficiently constructed if the corruptions are independent of the covariates $\bx_i$. This stage requires only $\calO(d \, \mathrm{polylog}( d))$ quadratic measurements, ensuring that the overall sample complexity is not increased.
\end{itemize}

\section{Solving The Phase Retrieval Problem}
\label{sec: solving phase retrieval}

Considerable effort has been devoted to solving the phase retrieval problem in the uncorrupted setting~\eqref{eq: standard phase retrieval}. The work by \citet{fienup1982phase} surveys many classical methods for addressing phase retrieval. These methods often involve an alternating minimization approach (different from ours), which alternates between recovering signal information and phase information. However, these methods typically either provide only local optimal solutions or lack performance guarantees altogether. Recently, \citet{netrapalli2013phase} proposed an alternating minimization approach and provided global convergence results using $\calO(d \, \mathrm{polylog}( d))$ samples. Modern approaches to the phase retrieval problem can be broadly placed into two main categories.

\subsubsection{Convex Formulations}
Many approaches formulate the problem as a convex optimization problem, often using a semidefinite programming (SDP) formulation. They provide performance guarantees for signal recovery by solving their proposed convex relaxations~\citep{candes2013phaselift, candes2014solving, chen2015exact, demanet2014stable, waldspurger2015phase}. While effective, these approaches suffer from heavy computational costs.
%, typically at least $\calO(d^3)$.
Recently, new convex relaxations have been developed that work in the domain of the original variables, avoiding the high computational complexity associated with SDP relaxation~\citep{goldstein2018phasemax, bahmani2017phase}. These methods offer a more computationally efficient alternative while still providing provable performance guarantees.

\subsubsection{Nonconvex Formulations}

A more natural formulation of the signal recovery problem in phase retrieval leads to a nonconvex program:
    \begin{align}
    \label{eq:nonconvex formulation}
        \hat{\btheta} = \arg\min_{\btheta \in \real^d} \; f(\btheta) \coloneqq  \frac{1}{4n}\sum_{i=1}^n \Big(y_i - \inner{\bx_i}{\btheta}^2\Big)^2~.
    \end{align}
Several methods have been proposed to solve this (or its nonsmooth variant) problem. Based on their initialization method, they can be further divided into two subcategories.
\begin{enumerate}
    \item \emph{Using spectral initialization:} The most notable approaches in this category utilize the Wirtinger flow algorithm or its variants~\citep{candes2015phase, chen2015solving, zhang2016reshaped, wang2017solving}. When initialized using a spectral method, they exhibit global convergence at a linear rate using only $\calO(d \, \mathrm{polylog}( d))$ measurements. However, the spectral initialization requires an eigenvalue decomposition, which can be computationally expensive, involving costs comparable to matrix inversion.
    \item \emph{Using random initialization}: \citet{sun2018geometric} demonstrated that the nonconvex loss landscape of the phase retrieval problem~\eqref{eq:nonconvex formulation} possesses special geometric properties. They proved that with $\calO(d \log^3 d)$ measurements, all the local minima of the loss function are also global minima, and all the saddle points are strict saddle points~\citep{ge2015escaping}. This property allows any saddle point-escaping algorithms, such as Hessian-based methods~\citep{nesterov2006cubic, sun2018geometric}, perturbed gradient descent~\citep{jin2017escape}, stochastic-gradient descent~\citep{ge2015escaping}, and normalized gradient descent~\citep{murray2017revisiting}, to converge to the global minima of the phase retrieval problem~\eqref{eq:nonconvex formulation} without requiring spectral initialization. However, these methods often incur high iteration complexity, typically at least $\calO(d^{2.5})$. Notable exceptions include the work by \citet{tan2023online}, which requires only $\calO(d \, \mathrm{polylog}( d))$ iterations with stochastic gradient updates (on the nonsmooth variant), and the work by \citet{chen2019gradient}, which needs only $\calO(\log d)$ iterations with full gradient updates.
\end{enumerate}

\section{Phase Retrieval With Corruptions}
\label{sec:phase retrieval with corruption}

The primary applications of phase retrieval \citep{walther1963question, millane1990phase} are susceptible to corrupted measurements due to instrument failures, Byzantine sensors, or other measurement errors~\citep{weller2015undersampled}. Therefore, the data generation model in \eqref{eq: standard phase retrieval} can be modified to account for these corruptions as follows:
\begin{align}
    \label{eq: corrupted phase retrieval}
    y_i = \inner{\bx_i}{\btheta^*}^2 + \eta_i, \quad i \in [n]~,
\end{align}
where corruption in $i$-th measurement is denoted as $\eta_i \in \real$ and drawn from an unknown distribution $\mathcal{P}_{\eta}$. We collect all the $\eta_i$'s in a vector $\bmeta \in \real^n$. Furthermore, we assume $\| \bmeta \|_0 = k$. The set $C^* \coloneqq \big\{ i \in [n]\, | \, \eta_i \ne 0 \big\}$ contains the indices of corrupted measurements. The task still remains to learn $\btheta^*$ from the quadratic measurements of the form $(\bx_i, y_i),\ \forall i \in [n]$.
We allow an adversary to arbitrarily corrupt $\epsilon = \frac{k}{n}$ fraction of responses, where $k =k_n$ is allowed to grow with $n$. We adopt a {\em strong corruption model}~\citep{bakshi2021robust} in our setting. Before introducing corruption, the adversary has full information about the measurements and the estimator.
\begin{definition}[Strong Corruption Model]
\label{def: strong corruption model}
    In this corruption model, the data generation process involves two steps:
    \begin{enumerate}
        \item Clean measurements $(\bx_i, y_i), i \in [n]$ are generated according to the noiseless data generation model given by~\eqref{eq: standard phase retrieval}. These measurements are collected in a set $S$.
        \item The adversary selects any subset $C^* \subseteq[n]$ of size $k$. For each $i \in C^*$, the adversary replaces $y_i$ with $y_i + \eta_i$, where $\eta_i$ is drawn from an unknown distribution $\mathcal{P}_{\eta}$. Notably, the adversary can choose $\eta_i$ that depends on the measurements $(\bx_i, y_i)$.
    \end{enumerate}
\end{definition}
This corruption model stands as the most stringent, encompassing a wide array of other response corruption models, including the Huber contamination model and sparse arbitrary outliers model~\citep{huang2023robust}. We note that a strong adversary could consistently set an $\epsilon$ proportion of measurements to be identical, even if those measurements originate from two distinct distributions following \eqref{eq: corrupted phase retrieval}. These distributions could differ by at most $1 - \epsilon$ in terms of total variation distance. Consequently, a high probability exact recovery guarantee becomes unattainable when $\epsilon$ is constant.

\begin{figure*}[!ht]
    \centering
    \begin{subfigure}{0.28\textwidth}
    \includegraphics[scale=0.29]{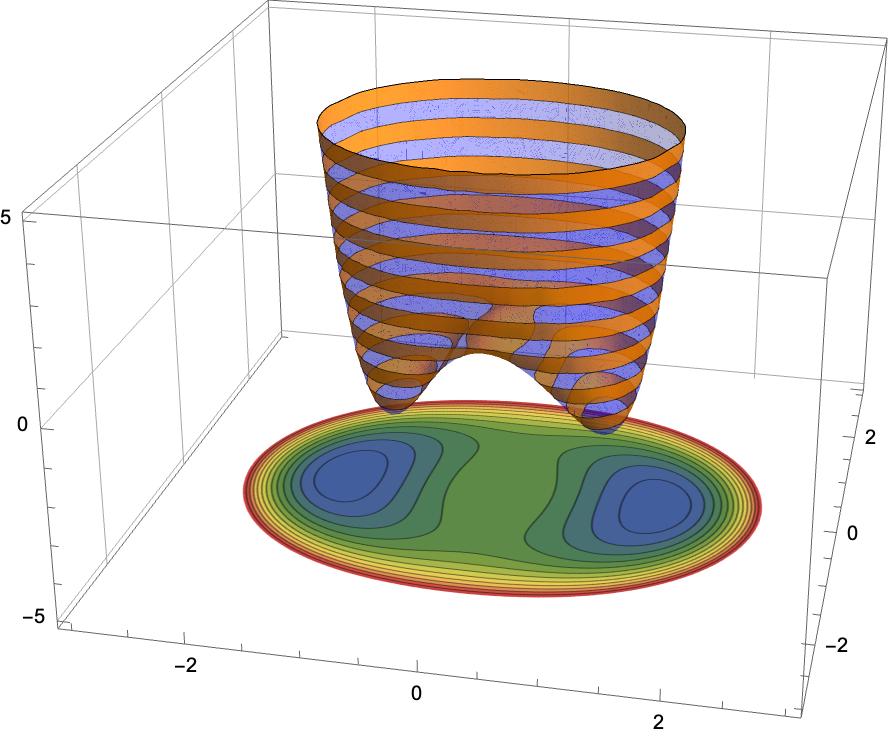}
        \caption{  No corruption }
        \label{fig: geom a}
    \end{subfigure}\hfill
    \begin{subfigure}{0.28\textwidth}
    \includegraphics[scale=0.29]{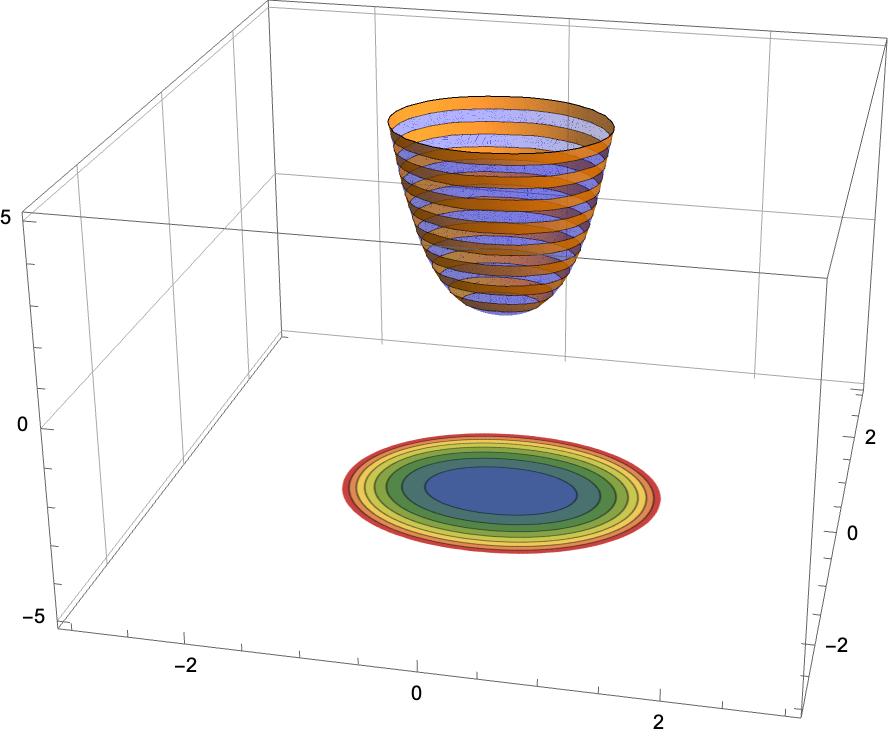}
        \caption{ Average corruption $\bar{\eta} \leq -3 \| \btheta^* \|^2$  }
        \label{fig: geom e}
    \end{subfigure}\hfill
     \begin{subfigure}{0.28\textwidth}
    \includegraphics[scale=0.29]{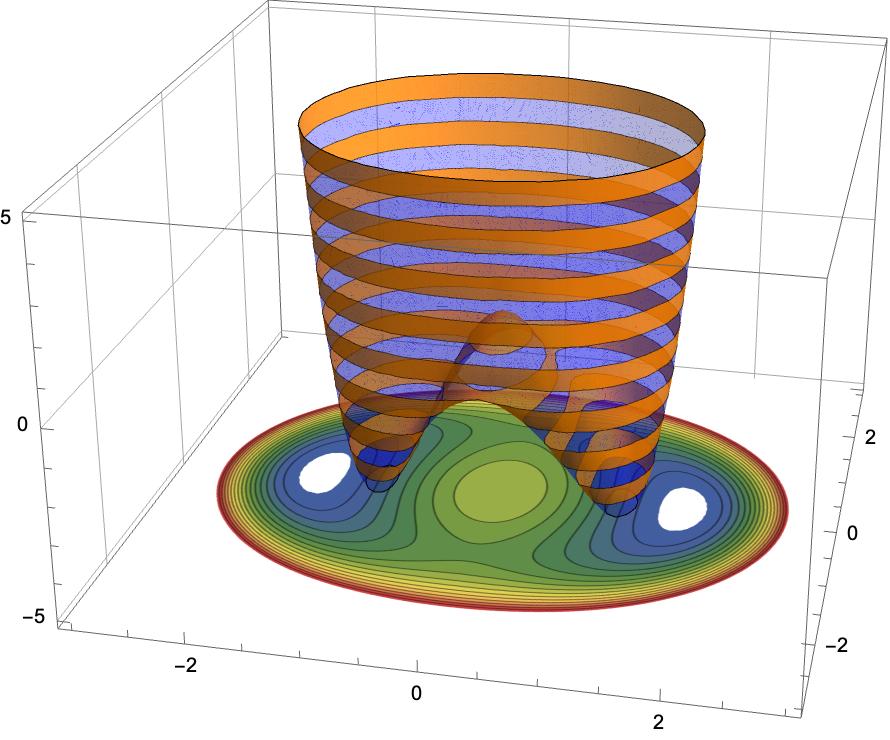}
        \caption{ Average corruption $\bar{\eta} > -3 \| \btheta^* \|^2$ }
        \label{fig: geom f}
    \end{subfigure}
    \caption{ Expected loss landscape for various corruption levels}
    \label{fig:  geometry of loss}
\end{figure*}
%The corruption $\eta_i$ introduced in the measurement $i$ is independent of covariate $\bx_i$ and the average corruption $\bar{\eta}$ is defined as $\frac{\sum_{i=1}^n \eta_i}{n}$.
\begin{restatable}[Impossibility with constant corruption proportion]{proposition}{propimpossibility}
\label{lem:impossibility with constant corruption}
If the measurements follow the data generation process~\eqref{eq: corrupted phase retrieval} with a corruption proportion $\epsilon > 0$, then for any estimator $\hat{\btheta}$ and any $\delta > 0$:
\begin{align}
    \label{eq: minimax constant corruption}
   %  \min\left\{\mathbb{P}\big[ \| \hat{\btheta} - \btheta^* \| \geq \delta \big], \mathbb{P}\big[ \| \hat{\btheta} +\btheta^* \| \geq \delta \big] \right\} \geq \frac{\epsilon}{2}~.
     \mathbb{P}\big[ d(\hat{\btheta}, \btheta^*) \geq \delta \big]  \geq \frac{\epsilon}{2}~.
\end{align}
\end{restatable}
Under the strong corruption model,
Proposition~\ref{lem:impossibility with constant corruption} states that estimation of $\btheta^*$ (up to its sign) fails with a constant probability when $\epsilon$ is a constant. Therefore, this work concentrates on the scenarios where the corruption proportion vanishes with $n$, i.e., $\epsilon=\epsilon_n\to 0$ as $n \to \infty$.

Efforts have been made to extend phase retrieval algorithms to settings involving corrupted measurements. \citet{hand2017phaselift} showed that PhaseLift~\citep{candes2013phaselift} is robust to a sufficiently small proportion of corrupted measurements, although the explicit proportion of corruption it can handle is not provided. Recently, \citet{huang2023robust} proposed a convex relaxation-based approach capable of handling approximately 11.85\% of corrupted measurements. They also demonstrated that this bound cannot be improved for their method. These approaches are computationally expensive due to their reliance on SDP relaxation. In contrast, \citet{zhang2016provable} proposed a computationally efficient approach using truncated Wirtinger flow, which can handle a small number of corrupted measurements with only $\calO(d \, \mathrm{polylog}( d))$ samples and $\calO(d^2 \, \mathrm{polylog}( d))$ iteration complexity. Existing works~\citep{hand2017phaselift,huang2023robust,zhang2016reshaped} in robust phase retrieval provide convergence guarantees for the estimator $\hat{\btheta}$ of the true parameter $\btheta^*$. However, these works do not explicitly address the dependency of convergence on the number of corrupted measurements $k$ relative to the total number of measurements $n$. Consequently, predicting the impact of an increasing proportion of corruption on the convergence rate remains difficult.
%Additionally, most of these studies do not delineate the largest corruption proportion that their algorithms can tolerate in terms of $k$ and $n$.
This leads to the following critical problem
of interest:
\begin{problem}[Robust Phase Retrieval]
\label{prob: robust phase retrieval}
     Can we propose a sample and computationally efficient algorithm that outputs an estimate $\hat{\btheta}$ of the true parameter $\btheta^*$ that allows us to identify a function $g(k, n)$ such that
     \begin{align*}
            d(\hat{\btheta}, \btheta^*) \leq g(k, n)~,
    \end{align*}
    and characterize regimes of  $\epsilon = \epsilon_n= {k}/{n}$ such that $g(k, n) \to 0$ as $n \to \infty$?
    % \begin{enumerate}
    %     \item identify a function $g(k, n)$ such that
    %     \begin{align*}
    %         d(\hat{\btheta}, \btheta^*) \leq g(k, n)~,
    %     \end{align*}
    %     \item and characterize regimes of  $\epsilon = \frac{k}{n}$ such that $g(k, n) \to 0$ as $n \to \infty$?
    % \end{enumerate}
\end{problem}

We provide an affirmative answer to Problem~\ref{prob: robust phase retrieval} by proposing an efficient alternating minimization-based algorithm.

\section{Alternating Minimization Algorithm}
\label{sec: alt min algorithm}

The main intuition behind the alternating minimization approach stems from a couple of key observations.
\subsection{Geometry of the loss function}
\label{subsec: geometry of the loss function}

Our first observation pertains to the geometrical properties of the objective function in the optimization problem~\eqref{eq:nonconvex formulation}. Although the optimization problem~\eqref{eq:nonconvex formulation} is nonconvex, it can be analyzed due to the favorable geometric properties of its objective function~\citep{sun2018geometric}. Indeed, numerous methods~\citep{ge2015escaping, nesterov2006cubic, sun2018geometric, jin2017escape, ge2015escaping, murray2017revisiting} discussed in Section~\ref{sec: solving phase retrieval} leverage this benign geometry to develop efficient algorithms with provable guarantees for solving~\eqref{eq:nonconvex formulation}. If the corrupted measurements do not significantly distort this geometry, we can still apply the methods outlined in Section~\ref{sec: solving phase retrieval} to handle the corruption. To illustrate, consider Figure~\ref{fig:  geometry of loss} where the expected objective function for problem~\eqref{eq:nonconvex formulation}, defined as $F(\btheta, \bmeta) \coloneqq \mathbb{E}_{\bx_1, \ldots, \bx_n} \big[ f(\btheta) \big]$, is graphically visualized for a two-dimensional example. When the corruptions $\eta_i$ are independently chosen from $\bx_i$, the shape of the loss function is influenced by the average corruption $\bar{\eta} = \frac{1}{n}\sum_{i=1}^n \eta_i$. Specifically, the function maintains a similar shape as the uncorrupted case when $\bar{\eta} > -3 \| \btheta^* \|^2$, and becomes convex when $\bar{\eta} \leq -3 \| \btheta^* \|^2$. These insights suggest that \eqref{eq:nonconvex formulation} can be solved even with corrupted measurements. Therefore, we introduce the \oracle{} (Algorithm~\ref{alg: ORACLE}), an oracle algorithm designed to solve the least squares formulation of the phase retrieval problem even when faced with corrupted measurements.
%Note that the \oracle{} can solve a nonconvex problem even with corrupted measurements.
Under mild assumptions regarding corruption, several methods described in Section~\ref{sec: solving phase retrieval} can be utilized for this purpose. We provide an efficient construction of one such oracle in Section~\ref{sec: constructing an oracle}.
\begin{algorithm}[!htb]
\caption{\textsc{LSQ-Phase-Oracle}}
\label{alg: ORACLE}
\textbf{Input}: $U \subseteq [n], S = \{ (\bx_1, y_1), \ldots, (\bx_n, y_n) \}$\\
\textbf{Output}: $\tilde{\btheta}$
\begin{algorithmic}[1] %[1] enables line numbers
\STATE  $\tilde{\btheta} = \arg\min_{ \btheta \in \real^d } \frac{1}{4 |U|} \sum_{i \in U} \Big(y_i - \inner{\bx_i}{\btheta}^2\Big)^2 $.
\STATE \textbf{return} $\tilde{\btheta}$.
\end{algorithmic}
\end{algorithm}

\subsection{Filtering the corrupted measurements}
\label{subsec: filtering the corrupted measurement}

Our second observation provides a counterbalance to our first observation. Despite the favorable geometric properties of the objective function, one cannot hope to reconstruct the true signal $\btheta^*$ by directly solving the nonconvex optimization problem~\eqref{eq:nonconvex formulation} in the presence of corrupted measurements. If we had prior knowledge of which measurements were uncorrupted, we could solve problem~\eqref{eq:nonconvex formulation} using only those uncorrupted measurements. However, in the absence of such information, we must reformulate it into another, potentially more challenging, nonconvex problem. To reduce the impact of obviously corrupted measurements, we initially preprocess the data by discarding measurements with negative $y_i$'s. We further trim the dataset by eliminating measurements with the largest $y_i$ values, ensuring the remaining set consists of $n - k$ measurements.  This refined set of measurements is denoted by $\tilde{S} \subset [n]$ with $|\tilde{S}| = n - k$. It is important to note that even after preprocessing step, $\tilde{S}$ might still contain up to $k$ corrupted measurements, although their values are restricted from being excessively large. We construct the following optimization problem after preprocessing:
\begin{align}
    \big(\hat{\btheta}, \hat{U}\big) = %\begin{matrix}
         &\mathop{\arg\min}\limits_{\btheta \in \real^d, U} f(U, \btheta) \coloneq \frac{1}{4 |U|}\sum_{i \in U} \Big(y_i - \inner{\bx_i}{\btheta}^2\Big)^2  \nonumber \\*
        &\text{such that } U \subset \tilde{S},\, |U| = n - 2 k\label{eq: nonconvex opt with corruption}
    %\end{matrix}
\end{align}
%Define $S_{< 0} \coloneqq \{ i \in [n] \,|\, y_i < 0\}$. Then $\tilde{S} \subseteq [n] \setminus S_{< 0}$  contains the indices of the remaining $n - k$ measurements after discarding the $k - |S_{< 0}|$ measurements with the largest $y_i$ values. This step helps eliminate the measurements with a large amount of corruption.
Ideally, we want to select  $n - 2k$ uncorrupted samples from $\tilde{S}$ and use them to estimate $\btheta^*$; this explains the term $n-2k $ in \eqref{eq: nonconvex opt with corruption}. A natural approach of solving problem~\eqref{eq: nonconvex opt with corruption} is to use an alternating minimization strategy. This method iterates between two steps: first, solving for $\btheta$ with a fixed $U$, and then updating $U$ based on the obtained $\btheta$. We formally present this approach as  \altmin{} in Algorithm~\ref{alg:alt min algorithm}.

\begin{algorithm}[!htb]
\caption{\textsc{ALT-MIN-PHASE}}
\label{alg:alt min algorithm}
\textbf{Input}: $S = \{ (\bx_1, y_1), \ldots, (\bx_n, y_n) \}$\\
\textbf{Parameters}: $k, \ \beta > 0$\\
\textbf{Output}: $\hat{\btheta}$ -- An estimate of $\btheta^*$
\begin{algorithmic}[1] %[1] enables line numbers
\STATE  $\btheta^1 = \bm{0} \in \real^d$.
\STATE \textbf{Preprocessing:}
\STATE Discard measurements with negative $y_i$'s
\STATE  From the remaining measurements discard measurements with the largest $y_i$ to construct $\tilde{S}$ with $|\tilde{S}| = n - k$
\FOR{$t=1, 2 \ldots$}
\STATE $U^t = \arg\min_{ U \subset \tilde{S}, |U| = n - 2k } \sum_{i \in U} f_i(\btheta^t)$,\\ where $f_i(\btheta) = \big(y_i - \inner{\bx_i}{\btheta}^2\big)^2$
\STATE $\btheta^{t+1} = \textsc{LSQ-Phase-Oracle}(U^t)$.
\IF {$\frac{1}{4|U^t|} \sum_{i \in U^t} \big(f_i(\btheta^t) - f_i(\btheta^{t+1}) \big) <  \beta$}
\STATE $\hat{U} = U^t, \; \hat{\btheta} = \btheta^t$.
\STATE STOP.
\ENDIF
\ENDFOR
\STATE \textbf{return} $\hat{\btheta}$.
\end{algorithmic}
\end{algorithm}

Leveraging tail bounds of maximum of the chi-squared random variables,  the preprocessing step effectively removes the measurements $i$ for which $\eta_i = \Omega(\log n)$ with high probability. Following this, the algorithm employs an alternating minimization strategy on the remaining measurements, utilizing \oracle{} as a subroutine to solve \eqref{eq:nonconvex formulation} with corruption. The process continues until the decrease of the objective function value is less than a certain predefined threshold $\beta > 0$ (to be fixed later). Intuitively, \altmin{} aims to produce a set $\hat{U}$ that contains indices of either uncorrupted measurements or corrupted measurements with minimal corruption. The key idea is that using such a $\hat{U}$ as input to the \oracle{} will provide a good estimate of $\btheta^*$. However, this guarantee is not obvious. Given the nonconvex nature of problem~\eqref{eq: nonconvex opt with corruption}, \altmin{} is likely to converge to a stationary point. Therefore, before presenting the theoretical guarantees for our approach, we demonstrate its practical efficacy through numerical experiments.

\subsection{Numerical Experiments}
\label{subsec: motivating example}

\begin{figure*}[!ht]
    \centering
    \begin{subfigure}{0.28\textwidth}
    \includegraphics[scale=0.53]{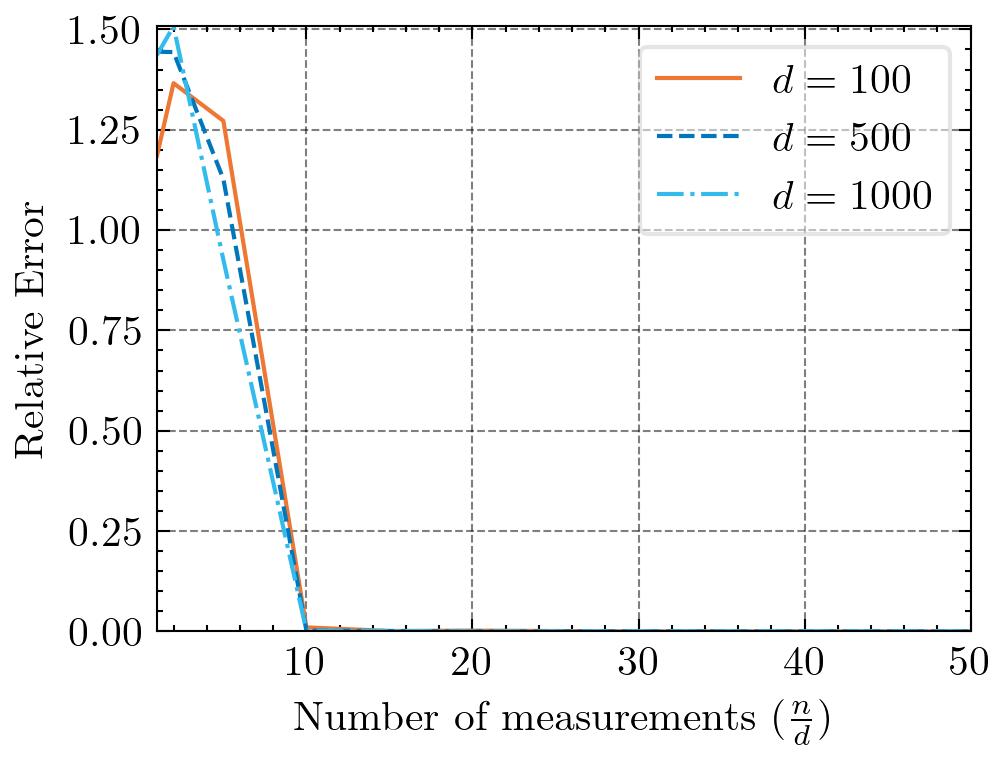}
        \caption{ $k = \sqrt{n} $}
        \label{fig: sqrt n}
    \end{subfigure}\hfill
    \begin{subfigure}{0.28\textwidth}
    \includegraphics[scale=0.53]{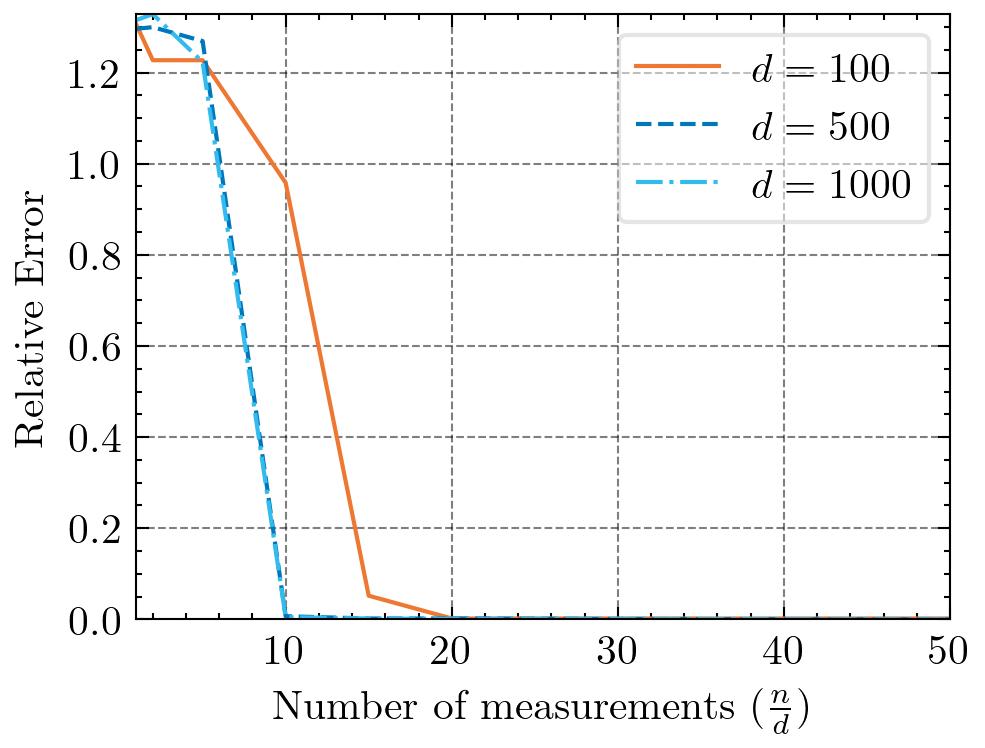}
        \caption{ $k = n^{\frac{2}{3}}$  }
        \label{fig: n 2 3}
    \end{subfigure}\hfill
     \begin{subfigure}{0.28\textwidth}
    \includegraphics[scale=0.53]{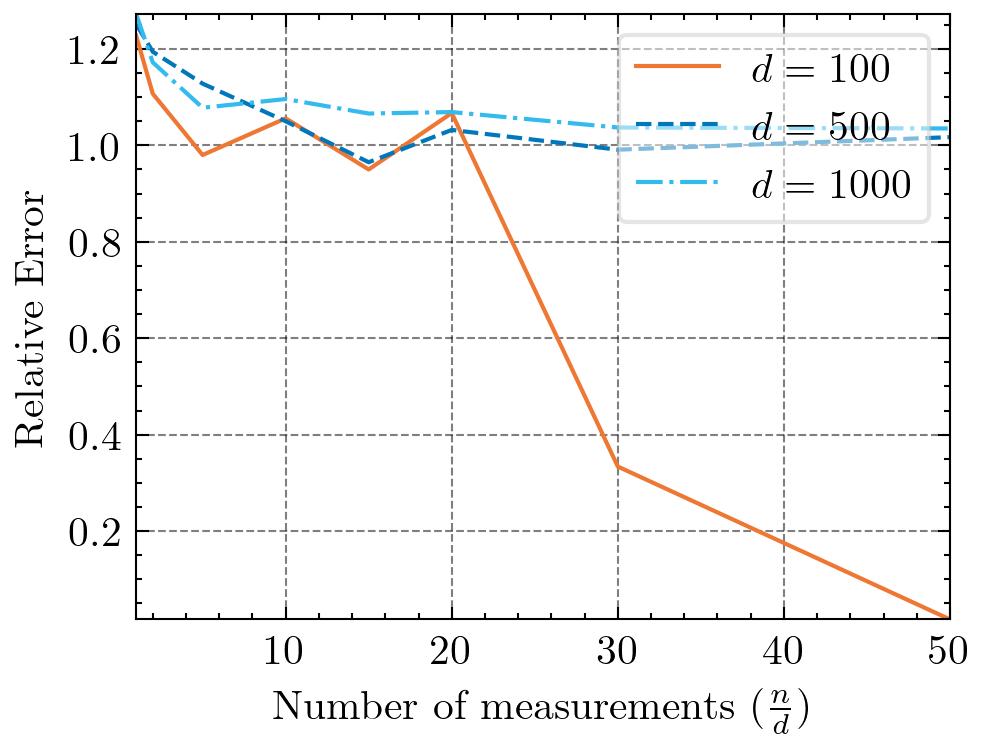}
        \caption{ $k = 0.25 n$ }
        \label{fig: const n}
    \end{subfigure}
    \caption{ Performance of \altmin{} (Algorithm~\ref{alg:alt min algorithm}) with varying degrees of corrupted measurements.}
    \label{fig:  convergence of alt min}
\end{figure*}

We evaluated the practicality of our approach through numerical experiments with varying degrees of corruption. We analyzed the effect of increasing the number of measurements on the relative error, defined as $\frac{d(\hat{\btheta}, \btheta^*)}{\| \btheta^* \|}$. The corruption  values $\eta_i$ were drawn uniformly at random from $[-5, 5]$. We used gradient descent with random initialization~\citep{chen2019gradient} as the oracle algorithm. Figure~\ref{fig: convergence of alt min} presents the results of our experiments for three different corruption proportions, $\frac{k}{n} \in \{\frac{1}{\sqrt{n}}, \frac{1}{n^{1/3}}, 0.25\}$. The plots correspond to $d \in \{100, 500, 1000\}$. We make the following observations:
\begin{enumerate}
    \item For $k \in \{\sqrt{n}, n^{\frac{2}{3}}\}$ (i.e., $\epsilon=o(1)$), Figures~\ref{fig: sqrt n} and \ref{fig: n 2 3} demonstrate that the alternating minimization approach accurately recovers $\btheta^*$ for each $d \in \{100, 500, 1000\}$.
    %with only $\calO(d)$ measurements.
    \item The proposed approach does not always yield a good estimate of $\btheta^*$ when there is a constant proportion of corruptions which is expected from Proposition~\ref{lem:impossibility with constant corruption}.
\end{enumerate}
Next, we present our theoretical results that explain these empirical observations.

\section{Main Theoretical Result}
\label{sec: main result}

In this section, we present a theoretical analysis of \altmin{}. While alternating minimization algorithms are known to converge to a stationary point for nonconvex problems, the resulting solution may not be close to the global minimum. Our analysis aims to characterize the properties of the converged stationary point and demonstrate its proximity to the global minimum.  Although our problem and setting are entirely different, our analysis follows a framework similar to \citet{chen2022online}, which addresses fixed-design linear regression with Huber corruption. The quartic nature of the objective function in \eqref{eq: nonconvex opt with corruption} imposes strong constraints on the handling of the spectral properties of the Hessian of the loss function, as well as on the concentration and tail bounds involved in the analysis. One immediate consequence of these constraints is the limited regime of corruption proportions that can be effectively managed by Algorithm~\ref{alg:alt min algorithm} with $n = \calO(d\, \mathrm{polylog}( d))$ measurements. We begin by characterizing this restricted regime. For a proportion of corruption $\epsilon$, we define
%$\Delta(k, n) \coloneqq \epsilon \sqrt{\log \epsilon^{-1}} \log^2(\epsilon n)$.
\begin{align}
\label{eq: Delta}
\Delta(k, n) \coloneqq \epsilon \sqrt{\log \epsilon^{-1}} \log^2(\epsilon n).
\end{align}
For brevity, we sometimes use $\Delta$ to denote $\Delta(k, n)$. We define the \emph{favorable corruption} regime $\mathcal{K}$ %\textcolor{red}{as the set of all increasing sequences $\{k_n\}_{n=1}^\infty$ satisfying}
as the set of all increasing sequences $\{k_n\}_{n=1}^\infty$ satisfying
\begin{align}
    \label{eq: favourable regime}
    \mathcal{K}\! \coloneqq \!\Big\{ \{k_n\}_{n=1}^\infty \subseteq \mathbb{N} \, | \, \lim_{n \to \infty} \Delta(k_n, n) \!=\! 0,\ \frac{k_n}{n} \!<\! \frac{1}{2} \Big\}~.
\end{align}
We remark that $\mathcal{K}$ captures all corruption regimes such that $k = \calO\big(n^{1 - p}\big)$ with $p \in (0, 1]$. On the other hand, it does not contain constant corruption proportion regimes in which $k = \Theta(n)$.
%This restriction can be mitigated if $n = \Omega(d^2)$ measurements are available, as detailed in Appendix~\ref{A}.
With this understanding, we are ready to state the main convergence guarantees for Algorithm~\ref{alg:alt min algorithm}.
%  as defined in~\eqref{eq: favourable regime} %
\begin{restatable}{theorem}{thmmain}
\label{thm:main result}
    Let $S = \{(\bx_i, y_i) \}_{i=1}^n$  be a set of measurements generated by the strong corruption model with corruption proportion $\epsilon = \frac{k}{n}$. We further assume that $k \in \mathcal{K}$. Let $n = \Omega\big( \frac{d \, \mathrm{polylog}(d) + \log(\frac{1}{\delta})}{\epsilon^2 \log (\frac{1}{\epsilon})} \big)$ for some $\delta \in (0, 1]$. With probability  at least $1 - \delta - \calO(\frac{1}{n})$, Algorithm~\ref{alg:alt min algorithm} with parameters $k$ and $\beta = \epsilon^2$ terminates within $\calO( \frac{\sum_{i=1}^n y_i^2}{4 (n - k)\epsilon^2} )$ iterations, and outputs an estimate $\hat{\btheta}$ such that  for some absolute constants $C_1, C_2, C_3 \!>\! 0$
    % \begin{align}
    %     d(\hat{\btheta}, \btheta^*) \leq \begin{cases} \psi(k, n, \bmeta) \epsilon, \quad \phantom{aaaaaa,} d(\hat{\btheta}, \btheta^*) \leq 1 \\
    %     \big( \sqrt{2} \psi(k, n, \bmeta) \big)^{\frac{1}{2}} \epsilon^{\frac{1}{2}}, \quad d(\hat{\btheta}, \btheta^*) > 1
    %     \end{cases}~,
    % \end{align}
    \begin{align}
        d(\hat{\btheta}, \btheta^*) \leq
        1.2  \max\Big\{\big( \psi(k, n, \bmeta) \big)^{\frac{1}{2}}, \psi(k, n, \bmeta)  \Big\} \sqrt{\epsilon}~,
    \end{align}
    where $\psi(k, n, \bmeta) = \frac{ \sqrt{ (C_1  + \max_i | \eta_i | ) (1 + \Delta) }   }{(1 - 3\epsilon)(C_2 - \Delta) - C_3 \Delta}$.
\end{restatable}

Before we present a proof sketch, several remarks are in order to understand Theorem~\ref{thm:main result}:
\begin{enumerate}
    \item It may appear that $\psi(k, n, \bmeta)$ can grow rapidly if the amount of corruption $\max_i |\eta_i |$ is large. However, it should be remembered that due to the preprocessing step performed by Algorithm~\ref{alg:alt min algorithm}, the growth of $\max_i |\eta_i|$ is not fast. In fact, it can be shown that its value is upper bounded by $\calO(\log n)$ with high probability. Since $\psi(k, n, \bmeta) = \calO(\sqrt{\log n})$, the convergence rate of the Algorithm~\ref{alg:alt min algorithm} is primarily determined by $\sqrt{\epsilon}$.
    %\item This implies that $\psi(k, n, \bmeta) \approx \log^{\frac{1}{4}} n$. Consequently, the convergence rate of the Algorithm~\ref{alg:alt min algorithm} is primarily determined by $\epsilon^r$ where the rate $r$ depends on $d(\hat{\btheta}, \btheta^*)$. Specifically, when $d(\hat{\btheta}, \btheta^*) \leq 1$ the rate of decay is faster. Conversely, when $d(\hat{\btheta}, \btheta^*) > 1$, the rate is slower.
    \item The dependence of $n$ on $\epsilon$ might initially seem counterintuitive, but it aligns with standard results in robust statistics~\citep{chen2022online,gao2020robust,diakonikolas2019efficient}.
    %This dependence is necessary to ensure the concentration bounds hold for $k = \epsilon n$ corrupted measurements.
    In practice, one can typically assume $\epsilon$ to be bounded away from zero by artificially considering some of the uncorrupted measurements as corrupted.
    \item The theoretical analysis of \altmin{} can be extended to phase retrieval with additive Gaussian noise, i.e., the uncorrupted measurements are $y_i + e_i$ where $e_i$ is an independent additive Gaussian noise. The fundamental analysis still applies, with the added consideration of concentration inequalities involving the noise variance.
    %We provide additional details in Section~\ref{sec:section}.
\end{enumerate}

In the following section, we offer a proof sketch highlighting the key elements of the proof for Theorem~\ref{thm:main result} when $k \in \calK$. Due to space limitations, the complete proof is provided in Appendix~\ref{sec: proof of thm: main result}.

\section{Proof Sketch of Theorem~\ref{thm:main result}}
\label{sec: proof sketch}

For a fixed set $U \subseteq [n]$, we define the loss function $f_U(\btheta)$ along with its gradient $\nabla f_U(\btheta)$
%and Hessian $\nabla^2 f_U(\btheta)$
as follows:
\begin{align}
    f_U(\btheta) &= \frac{1}{4 |U|}\sum_{i \in U} \Big(\inner{\bx_i}{\btheta}^2 - y_i\Big)^2 \\
    \nabla f_U(\btheta) &= \frac{1}{|U|}\sum_{i \in U} \Big(\inner{\bx_i}{\btheta}^2 - y_i \Big) \bx_i \bx_i^{\top} \btheta
\end{align}
%\\
%    \nabla^2 f_U(\btheta) &= \frac{1}{|U|}\sum_{i \in U} \Big(3 \inner{\bx_i}{\btheta}^2 - y_i \Big) \bx_i \bx_i^{\top}~.
Our analysis of \altmin{} is conducted in two stages. First, we demonstrate that the output $\hat{\btheta}$ from Algorithm~\ref{alg:alt min algorithm} is indeed a $\gamma$-approximate stationary point of $f_{\hat{U}}(\btheta)$ which is defined as:
\begin{align}
    \label{eq: gamma stationary point}
    \big\langle \nabla f_{\hat{U}}(\hat{\btheta}), \hat{\btheta} - \btheta^* \big\rangle \leq \gamma \| \hat{\btheta} - \btheta^* \|~.
\end{align}
After that, we establish that this approximate stationary point is close to the true signal $\btheta^*$ in terms of $d(\hat{\btheta}, \btheta^*)$.

\subsection{Convergence to An Approximate Stationary Point}

We begin by showing that the output of Algorithm~\ref{alg:alt min algorithm} is an approximate stationary point. Before we present the formal statement, we define the following quantity for some absolute constants $C_1, C_2$ and $C_3 > 0$:
\begin{align}
\label{eq:Lipschitz const}
    &  L(\hat{\btheta}, \btheta^*, \Delta, \bmeta) \coloneqq  \frac{1}{2}\Big( (C_1 + \Delta) \|\hat{\btheta} - \btheta^* \|^2 + (C_2 + \Delta) \|\hat{\btheta} - \btheta^* \| + (C_3 + \Delta) + \max_{i \in U} |\eta_i| \big( 1 + \Delta \big) \Big)~.
\end{align}
The following lemma shows that if we set $\gamma = 2 \sqrt{ L(\hat{\btheta}, \btheta^*, \Delta, \bmeta) } \epsilon$, the output of Algorithm~\ref{alg:alt min algorithm} is an $\gamma$-approximate point of $f_{\hat{U}}(\btheta)$.
\begin{restatable}{lemma}{lemconvergeapprox}
    \label{lem: convergence to approximate stationary point}
    If $n = \Omega\big( \frac{d \, \mathrm{polylog}(d)+\log(\frac{1}{\delta})}{\epsilon^2 \log (\frac{1}{\epsilon})} \big)$, then the output $\hat{\btheta}$ Algorithm~\ref{alg:alt min algorithm} is a $2 \sqrt{ L(\hat{\btheta}, \btheta^*, \Delta, \bmeta) } \epsilon$-stationary point of $f_{\hat{U}}(\btheta)$ with probability at least $1 - \delta - \calO(\frac{1}{n})$.
    % , i.e.,
    % \begin{align}
    %     \label{eq: approximate stationary point}
    %     \bigg\langle\nabla f_{\hat{U}}(\hat{\btheta}), \frac{\hat{\btheta} - \btheta^*}{\|\hat{\btheta} - \btheta^* \|} \bigg\rangle \leq 2 \sqrt{ L(\hat{\btheta}, \btheta^*, \Delta, \bmeta) } \epsilon~.
    % \end{align}
\end{restatable}
The convergence of the Algorithm~\ref{alg:alt min algorithm} depends on the spectral properties of $\nabla^2 f_{\hat{U}}(\btheta)$. Specifically, one can employ the descent lemma to study the convergence properties of the Algorithm~\ref{alg:alt min algorithm} if the spectral norm of $\nabla^2 f_{\hat{U}}(\btheta)$ is uniformly upper bounded by a constant for all $\btheta \in \real^d$. Unfortunately, due to the quartic nature of $f_{\hat{U}}(\btheta)$, this property does not hold in general. To overcome this challenge, we study the spectral properties of $\nabla^2 f_{\hat{U}}(\btheta)$ for $\btheta$ belonging to a specific set relevant to our setting. This leads to the following lemma.

\begin{restatable}{lemma}{lemspectralhessian}
    \label{lem: spectral property of Hessian}
    If $\btheta$ lies on the line segment connecting $\hat{\btheta}$ and $\btheta^*$ and $n = \Omega\big( \frac{d \, \mathrm{polylog}(d)+\log(\frac{1}{\delta})}{\epsilon^2 \log (\frac{1}{\epsilon})} \big)$ for some $\delta \in (0, 1]$, then with probability at least $1 - \delta - \calO(\frac{1}{n})$, %the following holds:
    \begin{align}
        \label{eq: spectral property of hessian}
        (\hat{\btheta} - \btheta^*)^{\top} \nabla^2 f_{\hat{U}}(\btheta) (\hat{\btheta} - \btheta^*) \leq 2 L(\hat{\btheta}, \btheta^*, \Delta, \bmeta ) \|\hat{\btheta} - \btheta^* \|^2.
    \end{align}
\end{restatable}
With Lemma~\ref{lem: spectral property of Hessian} in place, a modified version of the descent lemma follows immediately.
\begin{restatable}{lemma}{lemdescent}
    \label{lem: descent lemma}
    Let $\bar{\btheta}$ be any point on the line segment connecting $\hat{\btheta}$ and $\btheta^*$. Assume that the following properties hold:
    \begin{enumerate}
        \item $\inner{\nabla f_{\hat{U}}(\hat{\btheta})}{ \hat{\btheta} - \btheta^* } \geq \gamma \|\hat{\btheta} - \btheta^* \| > 0$ and
        \item $\frac{\hat{\btheta} - \btheta^*}{\|\hat{\btheta} - \btheta^* \|} \nabla^2 f_{\hat{U}}(\bar{\btheta}) \frac{\hat{\btheta} - \btheta^*}{\|\hat{\btheta} - \btheta^* \|} \leq 2 L(\hat{\btheta}, \btheta^*, \Delta, \bmeta )$.
    \end{enumerate}
    Then, there exists  $\btheta \in \real^d$ such that
    \begin{align}
        \label{eq: descent lemma}
        f_{\hat{U}}(\btheta) \leq f_{\hat{U}}(\hat{\btheta}) - \frac{\gamma^2}{4 L(\hat{\btheta}, \btheta^*, \Delta, \bmeta)}~.
    \end{align}
\end{restatable}
If the Algorithm~\ref{alg:alt min algorithm} stops, then it means that the decrease in the objective function is less than $\beta$.  By the contrapositive, it implies that $\beta \geq \frac{\gamma^2}{4 L(\hat{\btheta}, \btheta^*, \Delta, \bmeta)}$
% \begin{align}
%     \beta \geq \frac{\gamma^2}{4 L(\hat{\btheta}, \btheta^*, \Delta, \bmeta)}~,
% \end{align}
and
\begin{align}
    \inner{\nabla f_{\hat{U}}(\hat{\btheta})}{ \hat{\btheta} - \btheta^* } \leq 2 \sqrt{L(\hat{\btheta}, \btheta^*, \Delta, \bmeta) \beta} \|\hat{\btheta} - \btheta^* \|~.
\end{align}
%The result of Lemma~\ref{lem: convergence to approximate stationary point} follows by picking $\beta = \epsilon^2$.
Picking $\beta = \epsilon^2$ proves Lemma~\ref{lem: convergence to approximate stationary point}. Observe that Algorithm~\ref{alg:alt min algorithm} must terminate in a finite number of steps. For $\btheta = \bm{0}$, initial value of $f_{U}(\btheta)$ is not more than $\frac{\sum_{i=1}^n y_i^2}{4 (n - 2k)}$, and it decreases by at least $\beta > 0$ at each iteration. Since the objective function cannot assume negative values, the algorithm must terminate after at most $\frac{\sum_{i=1}^n y_i^2}{4 (n - 2 k)\epsilon^2}$ iterations.

\subsection{Proximity to the Ground Truth}
\label{subsec: closeness to the ground truth}

In this subsection, we build on the results from Lemma~\ref{lem: convergence to approximate stationary point} to show that $\hat{\btheta}$ is close to $\btheta^*$. First, observe that $d(\btheta, \btheta^*)$ can be treated as $\| \btheta - \btheta^* \|$ without loss of generality by possibly flipping the sign of $\btheta^*$. By setting $\gamma = 2 \sqrt{L(\hat{\btheta}, \btheta^*, \Delta, \bmeta)} \epsilon$, we can express the results from Lemma~\ref{lem: convergence to approximate stationary point} as follows:
% \begin{align}
%     \inner{\nabla f_{\hat{U}}(\hat{\btheta})}{ \frac{\hat{\btheta} - \btheta^*}{\|\hat{\btheta} - \btheta^* \|} } \leq \gamma~.
% \end{align}
% This leads to the following inequality:
\begin{align}
    \label{eq: approx stationary point 1}
    \frac{1}{|\hat{U}|}\sum_{i \in \hat{U}} \Big(\inner{\bx_i}{\hat{\btheta}}^2 - y_i \Big) \bx_i^{\top} \hat{\btheta} \big( \hat{\btheta} - \btheta^*\big)^{\top} \bx_i \leq \gamma \| \hat{\btheta} - \btheta^* \|~.
\end{align}
Recall that $|\hat{U}| < n$ and it contains both corrupted and uncorrupted measurements. Specifically, the measurements in $\hat{U}$ can be partitioned into two disjoint sets $ \hat{U} \cap U^*$ and $ \hat{U} \cap C^*  $ where $C^*$ is defined in Definition~\ref{def: strong corruption model} and $U^* =  [n] \setminus C^* $. Given these observations, we can rearrange terms to get:
%we can rearrange the terms in equation~\eqref{eq: approx stationary point 1} as follows:
\begin{align}
    \label{eq: approx stationary point 2}
    &\underbrace{\frac{1}{n}\sum_{i \in \hat{U} \cap U^*} \Big(\inner{\bx_i}{\hat{\btheta}}^2 - y_i \Big) \bx_i^{\top} \hat{\btheta} \big( \hat{\btheta} - \btheta^*\big)^{\top} \bx_i}_{ \zeta(\hat{\btheta}, \btheta^*, n, k) }  \leq \gamma \| \hat{\btheta} - \btheta^* \| \underbrace{- \frac{1}{n}\sum_{i \in \hat{U} \cap C^*} \Big(\inner{\bx_i}{\hat{\btheta}}^2 - y_i \Big) \bx_i^{\top} \hat{\btheta} \big( \hat{\btheta} - \btheta^*\big)^{\top} \bx_i}_{\xi(\hat{\btheta}, \btheta^*, n, k, \bmeta)}
\end{align}
Our aim is to provide a lower bound on $\zeta(\hat{\btheta}, \btheta^*, n, k)$  and an upper bound on $\xi(\hat{\btheta}, \btheta^*, n, k, \bmeta)$, both in terms of $\| \hat{\btheta} - \btheta^* \|$. Note that $\zeta(\hat{\btheta}, \btheta^*, n, k)$ does not contain any corrupted measurements. Therefore, we can simply replace $y_i$ with $\inner{\bx_i}{\btheta^*}^2$ using the data generation process in~\eqref{eq: standard phase retrieval}. The following lemma provides a lower bound on $\zeta(\hat{\btheta}, \btheta^*, n, k)$.

\begin{restatable}{lemma}{lemlowerbound}
    \label{lem: lower bound on term 1}
    If $n = \Omega\big( \frac{d \, \mathrm{polylog}(d)+\log(\frac{1}{\delta})}{\epsilon^2 \log (\frac{1}{\epsilon})} \big)$ for some $\delta \in (0, 1]$, then for some absolute constants $C_1, C_2$ and $C_3 > 0$, the following holds with probability at least $1 - \delta - \calO(\frac{1}{n})$:
    \begin{align}
        \label{eq:lower bound on Term 1}
        &\zeta(\hat{\btheta}, \btheta^*, n, k) \geq (1 - 3 \epsilon) \Big( \big( C_1 - \Delta \big)
        \| \hat{\btheta} - \btheta^* \|^4  + \big( C_2 - \Delta \big)
        \| \hat{\btheta} - \btheta^* \|^3 + \big( C_3 - \Delta \big)
        \| \hat{\btheta} - \btheta^* \|^2\Big)~.
    \end{align}
\end{restatable}

Extra care is needed to handle $\xi(\hat{\btheta}, \btheta^*, n, k, \bmeta)$ as it involves corrupted measurements. %We can bound this term by decomposing it into two separate components using the Cauchy-Schwartz inequality.
Using the Cauchy-Schwartz inequality,
\begin{align}
    \xi(\hat{\btheta}, \btheta^*, n, k, \bmeta) &\leq \Big( \underbrace{\frac{1}{n}\sum_{i \in \hat{U} \cap C^*} \big(\inner{\bx_i}{\hat{\btheta}}^2 - y_i \big)^2}_{\xi_1(\hat{\btheta}, \btheta^*, n, k, \bmeta) } \Big)^{\frac{1}{2}} \times \Big( \underbrace{\frac{1}{n}\sum_{i \in \hat{U} \cap C^*}
    \big( \bx_i^{\top} \hat{\btheta} \big( \hat{\btheta} - \btheta^*\big)^{\top} \bx_i \big)^2}_{\xi_2(\hat{\btheta}, \btheta^*, n, k) } \Big)^{\frac{1}{2}}~.
\end{align}
For a fixed $\hat{\btheta}$, Algorithm~\ref{alg:alt min algorithm} outputs $\hat{U}$ that yields the smallest loss. By removing the measurements belonging to $\hat{U} \cap U^*$ from both $\hat{U}$ and $U^*$, we obtain %derive the following upper bound on $\xi_1(\hat{\btheta}, \btheta^*, n, k, \bmeta)$:
\begin{align}
    \label{eq: upper bound on xi_1}
    \xi_1(\hat{\btheta}, \btheta^*, n, k, \bmeta) \leq \frac{1}{n}\sum_{i \in U^* \setminus \hat{U} } \big(\inner{\bx_i}{\hat{\btheta}}^2 - y_i \big)^2.
\end{align}
Note that the right-hand side of~\eqref{eq: upper bound on xi_1} contains terms that represent uncorrupted measurements. This allows us to provide an upper bound on $\xi_1(\hat{\btheta}, \btheta^*, n, k, \bmeta)$ that is independent of $\bmeta$. Similarly, $\xi_2(\hat{\btheta}, \btheta^*, n, k)$ does not involve any $y_i$ and thus remains unaffected by the corrupted measurements. We provide an upper bound on both terms in the following lemma.

\begin{restatable}{lemma}{lemupperbound}
    \label{lem: upper bound on xi}
    Let $n \!=\! \Omega\big( \frac{d \, \mathrm{polylog}(d) +\log(\frac{1}{\delta})}{\epsilon^2 \log (\frac{1}{\epsilon})} \big)$ for some $\delta\! \in\! (0, 1]$ and for some absolute constants $C_1, C_2$ and $C_3 > 0$, define
    \begin{align}
        \label{eq: upper bound expression}
        \upsilon(\hat{\btheta}, \btheta^*, n, k) &\coloneqq C_1 \Delta
        \| \hat{\btheta} - \btheta^* \|^4 + C_2 \Delta
        \| \hat{\btheta} - \btheta^* \|^3  + C_3 \Delta
        \| \hat{\btheta} - \btheta^* \|^2~.
    \end{align}
    Then, with probability at least $1 - \delta - \calO(\frac{1}{n})$:
    \begin{enumerate}
        \item $\xi_1(\hat{\btheta}, \btheta^*, n, k, \bmeta) \leq \upsilon(\hat{\btheta}, \btheta^*, n, k)$
        \item  $\xi_2(\hat{\btheta}, \btheta^*, n, k) \leq \upsilon(\hat{\btheta}, \btheta^*, n, k)$
        \item Consequently, $\xi(\hat{\btheta}, \btheta^*, n, k, \bmeta) \leq \upsilon(\hat{\btheta}, \btheta^*, n, k)$~.
    \end{enumerate}
\end{restatable}
Now we substitute $\gamma = 2 \sqrt{L(\hat{\btheta}, \btheta^*, \Delta, \bmeta)} \epsilon$ and combine  Lemmas~\ref{lem: lower bound on term 1} and~\ref{lem: upper bound on xi} which finally leads to  Theorem~\ref{thm:main result}.

\section{Constructing \oracle{}}
\label{sec: constructing an oracle}

Our theoretical results so far assume the existence of an oracle in Algorithm~\ref{alg: ORACLE}. Note that the optimization problem addressed by the \oracle{} is inherently nonconvex, even in the absence of corruption. However, in the absence of corruption, this problem can be efficiently solved due to its ``benign'' loss landscape. When arbitrary corruption is introduced under the strong corruption model, this benign geometry may not be preserved. To address this, we introduce an additional assumption in our corruption model to maintain the favorable geometric properties necessary for efficient optimization, at least with high probability.
\begin{assumption}
\label{assum: independent corruption}
    For each $i \in C^*$, the adversary draws $\eta_i \sim \mathcal{P}_{\eta}$ independently of $\bx_i$.
\end{assumption}
The corruption model under
Assumption~\ref{assum: independent corruption} is also known as the sparse arbitrary outliers model~\citep{huang2023robust}. Figure~\ref{fig: geometry of loss} illustrates the loss landscape of \eqref{eq: nonconvex opt with corruption} under Assumption~\ref{assum: independent corruption} for various levels of corruption. The solution provided by \oracle{}, when used in isolation, can be far from the ground truth $\btheta^*$. Therefore, to filter out the corrupted measurements, an alternating minimization procedure is necessary even with an oracle. While many approaches mentioned in Section~\ref{sec: solving phase retrieval} can be extended to work for our setting, we opt to use the gradient descent approach with random initialization \citep{chen2019gradient} as the oracle.

\subsection{Random Initialized Gradient Descent For Corrupted Measurements}

\citet{chen2019gradient} demonstrated that by analyzing the dynamics of the approximate state evolution of fixed step gradient descent updates, it is possible to show that randomly initialized gradient descent algorithm with a fixed step size converges linearly to the true solution, $\btheta^*$, in the uncorrupted case.  They carefully use a leave-one-out approach to handle the dependence between $\bx_i$ and the iterates. We argue that their method can also be extended to the corrupted case under Assumption \ref{assum: independent corruption}. We employ a modified version of their leave-one-out approach, with an adjusted initialization, and extend their results by examining various corruption scenarios separately. We present a modified version of \citet{chen2019gradient}'s algorithm in  Algorithm~\ref{alg: grad descent}.
Under Assumption~\ref{assum: independent corruption}, the geometry of the loss in equation~\eqref{eq:nonconvex formulation} is influenced by the average corruption, defined as $\bar{\eta} = \frac{1}{n} \sum_{i=1}^n \eta_i$. Analyzing the loss $F(\btheta, \bmeta)$, we observe that when $\bar{\eta} > -3 \| \btheta^* \|^2$, the loss landscape maintains a similar geometry to the no-corruption case but with a displaced global minimum occurring at $\pm \kappa \btheta^*$ where $\kappa \coloneqq \sqrt{1 + \frac{\bar{\eta}}{3 \| \btheta^* \|^2}}$. Conversely, when $\bar{\eta} \leq -3 \| \btheta^* \|^2$, the loss becomes convex with its minimum occurring at $\bm{0} \in \real^d$.
%The next lemma relates the global optimal solution for $F(\btheta, \bmeta)$ to that of the finite sample setting.
% \begin{lemma}
%     \label{lem: finite sample opt}
%     Under Assumption~\ref{assum: independent corruption} and for any $\rho > 0$, if $n = \Omega\bigg(\frac{ \big(d + \log \frac{1}{\delta} \big) \, \mathrm{polylog}(d) }{\rho^2}\bigg)$, then the following results hold with probability at least $1 - \delta$.
%     \begin{enumerate}
%         \item If $\bar{\eta} \leq -3 \| \btheta^* \|$, then $ f_U(\bm{0}) \leq \min_{\btheta \in \real^d} f_U(\btheta) + \rho~ $, and
%         \item if $\bar{\eta} > -3 \| \btheta^* \|$, then $ f_U(\kappa \btheta^*) \leq \min_{\btheta \in \real^d} f_U(\btheta) + \rho$.
%     \end{enumerate}
% \end{lemma}
%Lemma~\ref{lem: finite sample opt} demonstrates that when $\bar{\eta} \leq -3 \| \btheta^* \|$, the point $\bm{0}$ is almost an optimal solution for $\min_{\btheta \in \real^d} f_U(\btheta)$.
Algorithm~\ref{alg: grad descent} computes $\kappa_{\mathrm{sq}}$ which, in expectation, is equal to $\kappa^2$. A negative value of $\kappa_{\mathrm{sq}}$ indicates that $\bar{\eta} \leq -3 \| \btheta^* \|^2$. Therefore, Algorithm~\ref{alg: grad descent} returns $\bm{0}$ in this scenario. In the alternative case where $\bar{\eta} > -3 \| \btheta^* \|^2$,
%and by extending the analysis of \citet{chen2019gradient},
we present the following convergence result:
\begin{theorem}
    \label{thm: convergence of gradient descent}
    Under Assumption~\ref{assum: independent corruption}, if $n = \Omega\big(d\, \mathrm{polylog}(d)\big)$ and $\bar{\eta} > -3 \| \btheta^* \|^2$, then there exists  $\tilde{T} = \calO(\log d)$ such that with probability at least $1 - \calO\big(n^2 \exp(-1.5 d)\big) - \calO\big( n^{-9} \big)$, the iterates $\tilde{\btheta}^t$ of Algorithm~\ref{alg: grad descent} satisfy% the following properties:
    \begin{itemize}
        \item $\tilde{\btheta}^t$ converges linearly to $\kappa \btheta^*$ for all $t \geq \tilde{T}$, i.e.,
        \begin{align}
            \label{eq: linear convergence of gd}
            d(\tilde{\btheta}^t, \kappa \btheta^*) \leq \Big(1 - \frac{\mu}{2 } \| \btheta^* \|^2 \Big)^{t - \tilde{T}}  \| \btheta^* \|, \quad \forall t \geq \tilde{T}~.
        \end{align}
        \item The   ratio of the signal component $a_t \coloneqq | \inner{\btheta^t}{\kappa \btheta^*}|$ to the orthogonal component $b_t \coloneqq \| \btheta^t - \frac{\inner{\btheta^t}{\kappa \btheta^*}}{\| \btheta^* \|} \btheta^* \|$ obeys
        \begin{align}
            \frac{ a_t }{b_t} \geq \frac{c_2}{\sqrt{d\log d}} (1 + c_1 \mu^2)^t, \quad t = 0,1,\ldots
        \end{align}
        for some  universal constants $c_1, c_2 > 0$.
    \end{itemize}
\end{theorem}
%The results and analysis of Theorem~\ref{thm: convergence of gradient descent} largely build upon the work of \citet{chen2019gradient}. However, their approach requires some adjustments to accommodate corrupted measurements, which we discuss in detail in the appendix.
The second point of Theorem  \ref{thm: convergence of gradient descent} implies that the ratio of the signal strength to the strength of its orthogonal component grows as the iteration count increases. This ensures that the signal can be identified eventually as $t\to\infty$.
The proof of Theorem~\ref{thm: convergence of gradient descent} builds on \citet{chen2019gradient} but some changes are needed to handle the corruptions. %ed measurements. %Details are in the appendices.

\begin{algorithm}[t]
\caption{\textsc{Gradient Descent with Random Init}}
\label{alg: grad descent}
\textbf{Input}: $U \subseteq [n], S = \{ (\bx_1, y_1), \ldots, (\bx_n, y_n) \}$\\
\textbf{Parameters}: $\mu = \frac{c}{ \| \btheta^* \|^2}$ for  small $c > 0$ , $T = \Omega(\log d)$\\
\textbf{Output}: $\tilde{\btheta}$
\begin{algorithmic}[1] %[1] enables line numbers
\STATE  \textbf{Initialization:}
%\STATE $z_i = \sum_{j=1}^d x_{ij}^2, \forall i \in [n]$
\STATE $\kappa_{\mathrm{sq}} = \frac{1}{3} \left( \sqrt{2} \sqrt{ \frac{1}{m}\sum_{i=1}^m y_i^2 - \left( \frac{1}{m} \sum_{i=1}^m y_i  \right)^2 }  + \frac{1}{m} \sum_{i=1}^m y_i\right) $
\IF{$\kappa_{\mathrm{sq}} \leq 0$}
\STATE \textbf{return} $\tilde{\btheta} = 0$
\ELSE
\STATE $\tilde{\btheta}^0 = \sqrt{ \kappa_{\mathrm{sq}} } \bm{u}$ (where $\bm{u}$ is uniformly distributed over the unit sphere).
\FOR{$t=1,\ldots, T$}
\STATE $\tilde{\btheta}^{t+1} = \tilde{\btheta}^t - \mu \nabla_{\btheta} f_U(\tilde{\btheta}^t)$.
\ENDFOR
\STATE \textbf{return} $\tilde{\btheta} = \tilde{\btheta}^{T+1}$.
\ENDIF
\end{algorithmic}
\end{algorithm}
\vspace{-.1in}

\section{Conclusion}
\label{sec: conclusion}
%\paragraph{Conclusion.}
In this paper, we derived convergence rate guarantees for  \altmin{}, which is specifically designed for the phase retrieval problem under a strong corruption model. Our methodology facilitates signal recovery when $k = \mathcal{O}(n^{1 - p})$ for any $p \in (0, 1]$, with only $\Omega(d \, \mathrm{polylog}(d))$ measurements. Moreover, we provide an efficient construction of \oracle{} under a slightly less stringent corruption model.
Future research could explore extending our analysis to regimes with constant corruption proportions under Assumption~\ref{assum: independent corruption}. It would also be interesting to investigate oracles that do not require Assumption~\ref{assum: independent corruption}.

\bibliographystyle{apalike}
\bibliography{arxivphase}

\appendix

\section{Proof of Theorem~\ref{thm:main result}}
\label{sec: proof of thm: main result}

%\adnote{Has to be the first proof in appendix}

%\adnote{todo: replace repeating lemma/theorem statements using restatable}
\thmmain*
In this section, we provide the detailed proofs for the lemmas discussed in Section~\ref{sec: proof sketch}. Due to the rotational invariance of the Gaussian distribution, it is sufficient to demonstrate the results for $\btheta^* = [ 1, 0 , \ldots, 0]^{\top}$. This is similar to the approaches employed by \citet{chen2019gradient} and \citet{sun2018geometric}.

For a fixed $U \subset [n]$ with $\abs{U} = (1 - 2 \epsilon) n$, we restate the following definitions (along with the Hessian) from Section~\ref{sec: proof sketch}:
\begin{align}
    \label{eq: f_U grad and hessian}
    f_U(\btheta) &= \frac{1}{4 |U|}\sum_{i \in U} \Big(\inner{\bx_i}{\btheta}^2 - y_i\Big)^2 \\
    \nabla f_U(\btheta) &= \frac{1}{|U|}\sum_{i \in U} \Big(\inner{\bx_i}{\btheta}^2 - y_i \Big) \bx_i \bx_i^{\top} \btheta \\
    \nabla^2 f_U(\btheta) &= \frac{1}{|U|}\sum_{i \in U} \Big(3 \inner{\bx_i}{\btheta}^2 - y_i \Big) \bx_i \bx_i^{\top}~.
\end{align}

For ease of notation, when it does not introduce ambiguity, we reindex the elements in $U$ to $\{1, \ldots, m\}$, where $m = (1 - 2\epsilon) n$.

\subsection{Preprocessing Step}
Initially, observe that, by Corollary~\ref{cor: max of chi-squared random variables}, at most $k$ of the $y_i$'s can attain values of $\omega(\log n)$ with probability at least $1 - \newO{\frac{1}{n}}$. Consequently, the preprocessing step not only removes any $y_i$'s with negative values but also eliminates $y_i$'s that are of the order $\omega(\log n)$. This ensures that the remaining $\eta_i$'s are constrained to be of the order $\newO{\log n}$. Next, we prove Lemma~\ref{lem: spectral property of Hessian}.

\subsection{Proof of Lemma~\ref{lem: spectral property of Hessian}}
\lemspectralhessian*
% \begin{lemma}
%     If $\btheta$ lies on the line segment connecting $\hat{\btheta}$ and $\btheta^*$ and $n = \Omega\big( \frac{d \, \mathrm{polylog}(d)+\log(\frac{1}{\delta})}{\epsilon^2 \log (\frac{1}{\epsilon})} \big)$ for some $\delta \in (0, 1]$, then with probability at least $1 - \delta - \calO(\frac{1}{n})$, %the following holds:
%     \begin{align}
%         (\hat{\btheta} - \btheta^*)^{\top} \nabla^2 f_{\hat{U}}(\btheta) (\hat{\btheta} - \btheta^*) \leq 2 L(\hat{\btheta}, \btheta^*, \Delta, \bmeta ) \|\hat{\btheta} - \btheta^* \|^2.
%     \end{align}
% \end{lemma}
\begin{proof}
    We define $\bz = \frac{ \hat{\btheta} - \btheta^*}{ \| \hat{\btheta - \btheta^*} \| }$. Let $\bar{\btheta}$ be a point in the line-segment connecting $\hat{\btheta}$ and $\btheta^*$. This implies that there exists a $\bar{\lambda} \in [0, 1]$ such that
\begin{align}
    \bar{\btheta} = (1 - \bar{\lambda} ) \hat{\btheta} + \bar{\lambda} \btheta^*~.
\end{align}

Therefore,
\begin{align}
    \label{eq: bar theta and hat theta}
    \bar{\btheta} - \btheta^* = (1 - \bar{\lambda}) (\hat{\btheta} - \btheta^*)~.
\end{align}

Using the expression from \eqref{eq: f_U grad and hessian} and substituting $y_i$ from data generation model~\ref{def: strong corruption model}, we can write:

\begin{align}
    \bz^{\top} \nabla^2 f_{\hat{U}}(\bar{\btheta}) \bz &= \frac{1}{m} \sum_{i=1}^m \big( 3\inner{\bx_i}{\bar{\btheta}}^2 - \inner{\bx_i}{\btheta^*}^2  - \eta_i   \big) \bz^{\top} \bx_i \bx_i^{\top} \bz
\end{align}
After some algebraic manipulation, this can be rewritten as:
% \begin{align}
%     &= \frac{1}{n} \sum_{i=1}^n \big( 3\inner{x_i}{\bar{\theta}}^2 - 3 \inner{x_i}{\theta^*}^2 + 2 \inner{x_i}{\theta^*}^2  - \eta_i   \big) z^{\top} x_i x_i^{\top} z\\
%     &= \frac{1}{n} \sum_{i=1}^n \big( 3(\inner{x_i}{\bar{\theta}} -  \inner{x_i}{\theta^*}) (\inner{x_i}{\bar{\theta}} +  \inner{x_i}{\theta^*}) + 2 \inner{x_i}{\theta^*}^2  - \eta_i   \big) z^{\top} x_i x_i^{\top} z \\
%     &= \frac{1}{n} \sum_{i=1}^n \big( 3(\inner{x_i}{\bar{\theta}} -  \inner{x_i}{\theta^*}) (\inner{x_i}{\bar{\theta}} -  \inner{x_i}{\theta^*} + 2 \inner{x_i}{\theta^*}) + 2 \inner{x_i}{\theta^*}^2  - \eta_i   \big) z^{\top} x_i x_i^{\top} z \\
% \end{align}
\begin{align}
   \bz^{\top} \nabla^2 f_{\hat{U}}(\bar{\btheta}) \bz &= \frac{1}{m} \sum_{i=1}^m \big( 3(\inner{\bx_i}{\bar{\btheta}} -  \inner{\bx_i}{\btheta^*})^2 + 6 (\inner{\bx_i}{\bar{\btheta}} -  \inner{\bx_i}{\btheta^*})\inner{\bx_i}{\btheta^*}) + 2 \inner{\bx_i}{\btheta^*}^2  - \eta_i   \big) \bz^{\top} \bx_i \bx_i^{\top} \bz
\end{align}
Substituting the definition of $\bz$ and the result from \eqref{eq: bar theta and hat theta}, we get:
\begin{align}
   \bz^{\top} \nabla^2 f_{\hat{U}}(\bar{\btheta}) \bz &= \frac{1}{m} \sum_{i=1}^m \big( 3 (1 - \bar{\lambda})^2\inner{\bx_i}{\bz}^2 \| \hat{\btheta} - \btheta^* \|^2 + 6 (1 - \bar{\lambda}) \inner{\bx_i}{\bz} x_{i1} \| \hat{\btheta} - \btheta^*\| + 2 x_{i1}^2  - \eta_i   \big) \bz^{\top} \bx_i \bx_i^{\top} \bz \\
   &= \frac{1}{m} \sum_{i=1}^m \big( 3 (1 - \bar{\lambda})^2 \inner{\bx_i}{\bz}^4 \| \hat{\btheta} - \btheta^* \|^2 + 6 (1 - \bar{\lambda}) \inner{\bx_i}{\bz}^3 x_{i1} \| \hat{\btheta} - \btheta^*\| + 2 \inner{\bx_i}{\bz}^2 x_{i1}^2  - \eta_i \inner{\bx_i}{\bz}^2  \big)
\end{align}
Recall that $\Delta = \epsilon \sqrt{\log \epsilon^{-1}} \log^2(\epsilon n)$. Utilizing the results from Lemma~\ref{lem: concentration for uncorrupted} and Cauchy-Schwartz inequality, with probability at least $1 - \delta - \calO(\frac{1}{n})$:
\begin{align}
    \bz^{\top} \nabla^2 f_{\hat{U}}(\bar{\btheta}) \bz &\leq (1 - \bar{\lambda})^2 (C_{40} + \Delta) \| \hat{\btheta} - \btheta^* \|^2 + (1 - \bar{\lambda}) (C_{31} + \Delta) \| \hat{\btheta} - \btheta^* \| + (C_{22} + \Delta) + \max_i |\eta_i| (1 + \Delta)  \\
    &\leq (C_{40} + \Delta) \| \hat{\btheta} - \btheta^* \|^2 + (C_{31} + \Delta) \| \hat{\btheta} - \btheta^* \| + (C_{22} + \Delta) + \max_i |\eta_i| (1 + \Delta)  \\
    &= 2 L(\hat{\btheta}, \btheta^*, \Delta, \bmeta)~.
\end{align}

\end{proof}

\subsection{Proof of Lemma~\ref{lem: descent lemma}}
\lemdescent*
% \begin{lemma}
%     Let $\bar{\btheta}$ be any point on the line segment connecting $\hat{\btheta}$ and $\btheta^*$. Assume that the following properties hold:
%     \begin{enumerate}
%         \item $\inner{\nabla f_{\hat{U}}(\hat{\btheta})}{ \hat{\btheta} - \btheta^* } \geq \gamma \|\hat{\btheta} - \btheta^* \| > 0$ and
%         \item $\frac{\hat{\btheta} - \btheta^*}{\|\hat{\btheta} - \btheta^* \|} \nabla^2 f_{\hat{U}}(\bar{\btheta}) \frac{\hat{\btheta} - \btheta^*}{\|\hat{\btheta} - \btheta^* \|} \leq 2 L(\hat{\btheta}, \btheta^*, \Delta, \bmeta )$.
%     \end{enumerate}
%     Then, there exists  $\btheta \in \real^d$ such that
%     \begin{align}
%         f_{\hat{U}}(\btheta) \leq f_{\hat{U}}(\hat{\btheta}) - \frac{\gamma^2}{4 L(\hat{\btheta}, \btheta^*, \Delta, \bmeta)}~.
%     \end{align}
% \end{lemma}
\begin{proof}
        Consider a $\btheta = (1 - \lambda) \hat{\btheta} + \lambda \btheta^*$, where $\lambda \in [0, 1]$. Note that
    $\btheta - \hat{\btheta} = \lambda (\btheta^* - \hat{\btheta} )$. Using Taylor's theorem, we can write
    \begin{align}
        f_{\hat{U}}(\btheta) = f_{\hat{U}}(\hat{\btheta}) + \inner{\nabla f_{\hat{U}}(\hat{\btheta})}{\btheta - \hat{\btheta}} + \frac{1}{2} (\btheta - \hat{\btheta})^{\top} \nabla^2 f_{\hat{U}}(\bar{\btheta})  (\btheta - \hat{\btheta})~,
    \end{align}
    for some $\bar{\btheta}$ in the line-segment joining $\btheta$ and $\hat{\btheta}$. It follows that $\bar{\btheta}$ also lies in the line-segment joining $\hat{\btheta}$ and $\theta^*$. Thus,
    \begin{align}
        f_{\hat{U}}(\btheta) &\leq f_{\hat{U}}(\hat{\btheta}) - \lambda \gamma \| \hat{\btheta} - \btheta^* \| + L(\hat{\btheta}, \btheta^*, \Delta, \bmeta) \lambda^2 \| \hat{\btheta} - \btheta^* \|^2 \\
        &\leq f_{\hat{U}}(\hat{\btheta}) - \frac{\gamma^2}{4L(\hat{\btheta}, \btheta^*, \Delta, \bmeta)}~.
    \end{align}
    The final step follows by picking $\lambda = \frac{\gamma}{2L(\hat{\btheta}, \btheta^*, \Delta, \bmeta) \| \hat{\btheta} - \btheta^* \|}$.
\end{proof}

As outlined in Section~\ref{sec: proof sketch}, Lemma~\ref{lem: convergence to approximate stationary point} follows from combining the results of Lemma~\ref{lem: spectral property of Hessian} and Lemma~\ref{lem: descent lemma}. This establishes that $\hat{\btheta}$ is an $2 \sqrt{ L(\hat{\btheta}, \btheta^*, \Delta, \bmeta) } \epsilon$-stationary point of $f_{\hat{U}}(\btheta)$ with probability at least $1 - \delta - \calO(\frac{1}{n})$. We now proceed to analyze the proximity of $\hat{\btheta}$ to $\btheta^*$. For simplicity, we consider $d(\btheta, \btheta^*)$ as $\| \btheta - \btheta^* \|$, without loss of generality, by potentially flipping the sign of $\btheta^*$. Following the arguments from Section~\ref{sec: proof sketch}, our first task is to prove Lemma~\ref{lem: lower bound on term 1}.

\subsection{Proof of Lemma~\ref{lem: lower bound on term 1}}
\lemlowerbound*
% \begin{lemma}
%     If $n = \Omega\big( \frac{d \, \mathrm{polylog}(d)+\log(\frac{1}{\delta})}{\epsilon^2 \log (\frac{1}{\epsilon})} \big)$ for some $\delta \in (0, 1]$, then for some absolute constants $C_1, C_2$ and $C_3 > 0$, the following holds with probability at least $1 - \delta - \calO(\frac{1}{n})$:
%     \begin{align}
%         &\zeta(\hat{\btheta}, \btheta^*, n, k) \geq (1 - 3 \epsilon) \Big( \big( C_1 - \Delta \big)
%         \| \hat{\btheta} - \btheta^* \|^4  + \big( C_2 - \Delta \big)
%         \| \hat{\btheta} - \btheta^* \|^3 + \big( C_3 - \Delta \big)
%         \| \hat{\btheta} - \btheta^* \|^2\Big)~.
%     \end{align}
% \end{lemma}
\begin{proof}
    Recall that
    \begin{align}
        \zeta(\hat{\btheta}, \btheta^*, n, k) = \frac{1}{n} \sum_{i \in U^* \cap \hU} \big( \inner{\bx_i}{\hat{\btheta}}^2 - \inner{\bx_i}{\btheta^*}^2    \big)  \bx_i^{\top} \hat{\btheta} (\hat{\btheta} - \btheta^*)^{\top} \bx_i
    \end{align}
    We rewrite it in the following way.
    \begin{align}
        &\frac{1}{n} \sum_{i \in U^* \cap \hU} \big( \inner{\bx_i}{\hat{\btheta}}^2 - \inner{\bx_i}{\btheta^*}^2    \big)  \bx_i^{\top} \hat{\btheta} (\hat{\btheta} - \btheta^*)^{\top} \bx_i \\
        &= \frac{1}{n} \sum_{i \in U^* \cap \hU} \Bigg( \inner{\bx_i}{ \hat{\btheta} - \btheta^* }^4 + 4 \inner{\bx_i}{ \hat{\btheta} - \btheta^*  }^2 \inner{\bx_i}{\btheta^*}^2 +  4 \inner{\bx_i}{ \hat{\btheta} - \btheta^*  }^3 \inner{\bx_i}{\btheta^*}  \Bigg)\\
        &= \frac{1}{n} \sum_{i \in U^* \cap \hU} \Bigg( \inner{\bx_i}{\bz}^4 \| \btheta^* - \hat{\btheta} \|^4 + 4 \inner{\bx_i}{ \bz}^2 x_{i1}^2 \| \btheta^* - \hat{\btheta} \|^2 +  4 \inner{\bx_i}{ \bz}^3 x_{i1} \| \btheta^* - \hat{\btheta} \|^3  \Bigg)~,
    \end{align}
    where the last equality is by defining $\bz = \frac{\hat{\btheta} - \btheta^*}{\| \hat{\btheta} - \btheta^* \|}$. Note that $|U^* \cap \hat{U}| \geq (1 - 3\epsilon) n$ and using the result from Lemma~\ref{lem: concentration for uncorrupted}, we can write that with probability at least $1 - \delta - \newO{\frac{1}{n}}$,
    \begin{align}
        \label{eq: lower bound on LHS}
        \zeta(\hat{\btheta}, \btheta^*, n, k) \geq (1 - 3\epsilon) \Big( \big( C_{40} - \Delta \big) \| \btheta^* - \hat{\btheta} \|^4 + \big( C_{22} - \Delta \big) \| \btheta^* - \hat{\btheta} \|^2 + \big( C_{31} - \Delta \big) \| \btheta^* - \hat{\btheta} \|^3 \Big)
    \end{align}
\end{proof}

\subsection{Proof of Lemma~\ref{lem: upper bound on xi}}

Recall that
\begin{align}
    \xi(\hat{\btheta}, \btheta^*, n, k, \bmeta) = - \frac{1}{n}\sum_{i \in \hat{U} \cap C^*} \Big(\inner{\bx_i}{\hat{\btheta}}^2 - y_i \Big) \bx_i^{\top} \hat{\btheta} \big( \hat{\btheta} - \btheta^*\big)^{\top} \bx_i
\end{align}

Using the Cauchy-Schwartz inequality,
\begin{align}
    \xi(\hat{\btheta}, \btheta^*, n, k, \bmeta) &\leq \Big( \underbrace{\frac{1}{n}\sum_{i \in \hat{U} \cap C^*} \big(\inner{\bx_i}{\hat{\btheta}}^2 - y_i \big)^2}_{\xi_1(\hat{\btheta}, \btheta^*, n, k, \bmeta) } \Big)^{\frac{1}{2}} \times \Big( \underbrace{\frac{1}{n}\sum_{i \in \hat{U} \cap C^*}
    \big( \bx_i^{\top} \hat{\btheta} \big( \hat{\btheta} - \btheta^*\big)^{\top} \bx_i \big)^2}_{\xi_2(\hat{\btheta}, \btheta^*, n, k) } \Big)^{\frac{1}{2}}~.
\end{align}
For a fixed $\hat{\btheta}$, Algorithm~\ref{alg:alt min algorithm} outputs $\hat{U}$ that yields the smallest loss. This implies that
\begin{align}
    \frac{1}{n}\sum_{i \in \hat{U} } \big(\inner{\bx_i}{\hat{\btheta}}^2 - y_i \big)^2 \leq \frac{1}{n}\sum_{i \in U^*  } \big(\inner{\bx_i}{\hat{\btheta}}^2 - y_i \big)^2~.
\end{align}
By removing the measurements belonging to $\hat{U} \cap U^*$ from both $\hat{U}$ and $U^*$, we obtain %derive the following upper bound on $\xi_1(\hat{\btheta}, \btheta^*, n, k, \bmeta)$:
\begin{align}
    \label{eq: upper bound on xi_1_1}
    \xi_1(\hat{\btheta}, \btheta^*, n, k, \bmeta) = \frac{1}{n}\sum_{i \in \hat{U} \cap C^*} \big(\inner{\bx_i}{\hat{\btheta}}^2 - y_i \big)^2 \leq \frac{1}{n}\sum_{i \in U^* \setminus \hat{U} } \big(\inner{\bx_i}{\hat{\btheta}}^2 - y_i \big)^2.
\end{align}
Note that the right-hand side of~\eqref{eq: upper bound on xi_1_1} contains terms that represent uncorrupted measurements. This allows us to provide an upper bound on $\xi_1(\hat{\btheta}, \btheta^*, n, k, \bmeta)$ that is independent of $\bmeta$. Similarly, $\xi_2(\hat{\btheta}, \btheta^*, n, k)$ does not involve any $y_i$ and thus remains unaffected by the corrupted measurements. Now we are ready to prove Lemma~\ref{lem: upper bound on xi}.
\lemupperbound*
% \begin{lemma}
%     Let $n \!=\! \Omega\big( \frac{d \, \mathrm{polylog}(d) +\log(\frac{1}{\delta})}{\epsilon^2 \log (\frac{1}{\epsilon})} \big)$ for some $\delta\! \in\! (0, 1]$ and for some absolute constants $C_1, C_2$ and $C_3 > 0$, define
%     \begin{align}
%         \upsilon(\hat{\btheta}, \btheta^*, n, k) &\coloneqq C_1 \Delta
%         \| \hat{\btheta} - \btheta^* \|^4 + C_2 \Delta
%         \| \hat{\btheta} - \btheta^* \|^3 \ + C_3 \Delta
%         \| \hat{\btheta} - \btheta^* \|^2~.
%     \end{align}
%     Then, with probability at least $1 - \delta - \calO(\frac{1}{n})$:
%     \begin{enumerate}
%         \item $\xi_1(\hat{\btheta}, \btheta^*, n, k, \bmeta) \leq \upsilon(\hat{\btheta}, \btheta^*, n, k)$
%         \item  $\xi_2(\hat{\btheta}, \btheta^*, n, k) \leq \upsilon(\hat{\btheta}, \btheta^*, n, k)$
%         \item Consequently, $\xi(\hat{\btheta}, \btheta^*, n, k, \bmeta) \leq \upsilon(\hat{\btheta}, \btheta^*, n, k)$~.
%     \end{enumerate}
% \end{lemma}
\begin{proof}
We start by bounding $\xi_1(\hat{\btheta}, \btheta^*, n, k, \bmeta)$.

\subsubsection{Upper bound $\xi_1$.}
We showed in \eqref{eq: upper bound on xi_1_1} that,
\begin{align}
    \frac{1}{n} \sum_{i \in \hU \cap C^* } \big( \inner{\bx_i}{\hat{\btheta}}^2 - y_i   \big)^2 &\leq \frac{1}{n} \sum_{i \in U^* \setminus \hU} \big( \inner{\bx_i}{\hat{\btheta}}^2 - \inner{\bx_i}{\btheta^*}^2   \big)^2 \\
    &= \frac{1}{n} \sum_{i \in U^* \setminus \hU} \Big( \inner{\bx_i}{ \hat{\btheta} - \btheta^*  }^4 + 4 \inner{\bx_i}{ \hat{\btheta} - \btheta^*  }^2 \inner{\bx_i}{\btheta^*}^2  +  4 \inner{\bx_i}{ \hat{\btheta} - \btheta^* }^3 \inner{\bx_i}{\btheta^*}   \Big) \\
    &= \frac{1}{n} \sum_{i \in U^* \setminus \hU} \Big( \inner{\bx_i}{\bz}^4 \| \btheta^* - \hat{\btheta} \|^4 + 4 \inner{\bx_i}{ \bz}^2 x_{i1}^2 \| \btheta^* - \hat{\btheta} \|^2 +  4 \inner{\bx_i}{\bz}^3 x_{i1} \| \btheta^* - \hat{\btheta} \|^3  \Big)
\end{align}
where we substitute $\bz = \frac{\hat{\btheta} - \btheta^*}{\| \hat{\btheta} - \btheta^* \|}$ in the last equation. Note that $\abs{U^* \setminus \hU} \leq \epsilon n$. Using Lemma~\ref{eq: concentration for corrupted}, we get
\begin{align}
    \label{eq: bound on xi_1 1}
    \xi_1(\hat{\btheta}, \btheta^*, n, k, \bmeta) &\leq
    D_{40} \Delta \| \btheta^* - \hat{\btheta} \|^4 + D_{22} \Delta \| \btheta^* - \hat{\btheta} \|^2 + D_{31} \Delta \| \btheta^* - \hat{\btheta} \|^3~,
\end{align}
with probability at least $1 - \calO(\delta) - \newO{\frac{1}{n}}$ for some absolute constants $D_{40}, D_{22}, D_{31} > 0$. Next, we provide an upper bound on $\xi_2(\hat{\btheta}, \btheta^*, n, k)$.

\paragraph{Upper bound on $\xi_2(\hat{\btheta}, \btheta^*, n, k)$.}
By simple algebraic manipulation, we can write:
\begin{align}
    \xi_2(\hat{\btheta}, \btheta^*, n, k)
    &= \frac{1}{n} \sum_{i \in C^* \cap \hU} \Big( \inner{\bx_i}{ \hat{\btheta} - \btheta^*  }^4 +  \inner{\bx_i}{ \hat{\btheta} - \btheta^*  }^2 \inner{\bx_i}{\btheta^*}^2  +  2 \inner{\bx_i}{ \hat{\btheta} - \btheta^* }^3 \inner{\bx_i}{\btheta^*}   \Big) \\
    &= \frac{1}{n} \sum_{i \in C^* \cap \hU} \Big( \inner{\bx_i}{\bz}^4 \| \btheta^* - \hat{\btheta} \|^4 +  \inner{\bx_i}{ \bz}^2 \bx_{i1}^2 \| \btheta^* - \hat{\btheta} \|^2 +  2 \inner{\bx_i}{\bz}^3 x_{i1} \| \btheta^* - \hat{\btheta} \|^3  \Big)~,
\end{align}
where the last equation again uses $\bz = \frac{\hat{\btheta} - \btheta^*}{\| \hat{\btheta} - \btheta^* \|}$. Observe that $| C^* \cap \hU | \leq \epsilon n$. Using Lemma~\ref{eq: concentration for corrupted}, we get
\begin{align}
    \label{eq: bound on xi_2 1}
    \xi_2(\hat{\btheta}, \btheta^*, n, k) &\leq
    E_{40} \Delta \| \btheta^* - \hat{\btheta} \|^4 + E_{22} \Delta \| \btheta^* - \hat{\btheta} \|^2 + E_{31} \Delta \| \btheta^* - \hat{\btheta} \|^3~,
\end{align}
with probability at least $1 - \calO(\delta) - \newO{\frac{1}{n}}$ for some absolute constants $E_{40}, E_{22}, E_{31} > 0$.

The final upper bound on $\xi(\hat{\btheta}, \btheta^*, n, k, \bmeta)$ simply combines \eqref{eq: bound on xi_1 1} and \eqref{eq: bound on xi_2 1}.
\end{proof}

We are now ready to put everything together. Recall that,
\begin{align}
    \zeta(\hat{\btheta}, \btheta^*, n, k) \leq \gamma \| \hat{\btheta} - \btheta^* \| + \xi(\hat{\btheta}, \btheta^*, n, k, \bmeta)
\end{align}

Recall from Lemma~\ref{lem: convergence to approximate stationary point} that $\gamma = 2 \sqrt{L(\hat{\btheta}, \btheta^*, \Delta, \bmeta)} \epsilon$ and
\begin{align}
     2 L(\hat{\btheta}, \btheta^*, \Delta, \bmeta) &= (C_{40} + \Delta) \| \hat{\btheta} - \btheta^* \|^2 + (C_{31} + \Delta) \| \hat{\btheta} - \btheta^* \| + (C_{22} + \Delta) + \max_i |\eta_i| (1 + \Delta)~.
\end{align}

Substituting the lower bound on $\zeta(\hat{\btheta}, \btheta^*, n, k)$ and upper bound on $\xi(\hat{\btheta}, \btheta^*, n, k, \bmeta)$, we get:
\begin{align}
     &(1 - 3\epsilon) \Big( \big( C_{40} - \Delta \big) \| \btheta^* - \hat{\btheta} \|^3 + \big( C_{22} - \Delta \big) \| \btheta^* - \hat{\btheta} \| + \big( C_{31} - \Delta \big) \| \btheta^* - \hat{\btheta} \|^2 \Big) \\
     &\quad\quad\quad \leq   \sqrt{ 2 \big( (C_{40} + \Delta) \| \hat{\btheta} - \btheta^* \|^2 + (C_{31} + \Delta) \| \hat{\btheta} - \btheta^* \| + (C_{22} + \Delta) + \max_i |\eta_i| (1 + \Delta)  \big)} \epsilon  \\
     &\quad\quad\quad\quad + F_{40} \Delta \| \btheta^* - \hat{\btheta} \|^3 + F_{22} \Delta \| \btheta^* - \hat{\btheta} \| + F_{31} \Delta \| \btheta^* - \hat{\btheta} \|^2~,
\end{align}
where $F_{40}, F_{22}, F_{31} > 0$ are some appropriately chosen constants. To simplify this expression,  we consider two regimes.
\begin{enumerate}
    \item When $\| \btheta^* - \hat{\btheta} \| \leq 1$, for some absolute constants $H, G, I > 0$
    \begin{align}
        (1 - 3\epsilon) \big( I - \Delta \big) \| \btheta^* - \hat{\btheta} \|  &\leq \sqrt{ (H  + \max_i | \eta_i | ) (1 + \Delta) } \epsilon  + G \Delta \| \btheta^* - \hat{\btheta} \|~.
    \end{align}
    We can further simplify this to
    \begin{align}
        \| \btheta^* - \hat{\btheta} \| &\leq  \frac{ \sqrt{ (H  + \max_i | \eta_i | ) (1 + \Delta) }   }{(1 - 3\epsilon)(I - \Delta) - G\Delta} \epsilon
    \end{align}

    \item Similarly, when $\| \btheta^* - \hat{\btheta} \| \geq 1$,
    \begin{align}
        (1 - 3\epsilon) \big( C_{22} - \Delta \big) \| \btheta^* - \hat{\btheta} \|^3   &\leq \sqrt{ 2 \big( (C_{40} + \Delta)  + (C_{31} + \Delta)  + (C_{22} + \Delta) + \max_i |\eta_i| (1 + \Delta)  \big)} \| \hat{\btheta} - \btheta^* \| \epsilon  \\
        &\quad + G \Delta \| \btheta^* - \hat{\btheta} \|^3
    \end{align}

Upon further simplification and for some absolute constant $H, G, I > 0$, we get:
    \begin{align}
        &(1 - 3\epsilon) \big( I - \Delta \big) \| \btheta^* - \hat{\btheta} \|^3   \leq \sqrt{ 2 \big( (H + \max_i |\eta_i|) (1 + \Delta)  \big)} \| \hat{\btheta} - \btheta^* \| \epsilon  + G \Delta \| \btheta^* - \hat{\btheta} \|^3
    \end{align}
This leads to:
\begin{align}
    \| \btheta^* - \hat{\btheta} \| \leq \sqrt{ \frac{\sqrt{ 2 \big( (H + \max_i |\eta_i|) (1 + \Delta)  \big)} }{ (1 - 3\epsilon) (I - \Delta) - G\Delta } } \sqrt{\epsilon}~.
\end{align}
\end{enumerate}

Combining the results from both the regimes, we get
\begin{align}
    d(\hat{\btheta}, \btheta^*) \leq
    1.2  \max\Big\{\big( \psi(k, n, \bmeta) \big)^{\frac{1}{2}}, \psi(k, n, \bmeta)  \Big\} \sqrt{\epsilon}~,
\end{align}
where $\psi(k, n, \bmeta) = \frac{ \sqrt{ (H  + \max_i | \eta_i | ) (1 + \Delta) }   }{(1 - 3\epsilon)(I - \Delta) - G \Delta}$.

\section{Auxiliary Lemmas}
\label{sec: auxiliary lemmas}

In this section, we collect several auxiliary lemmas that are utilized throughout various parts of this paper.

\begin{lemma}[Concentration of the max of Gaussian random variables]
\label{lem: concentration of max of gaussians}
    Let $a_i \sim \normal(0, 1), i \in [n]$ be the $n$ i.i.d. Gaussian random variables. Define $a \coloneqq \max_{i \in [n]} a_i$. Then the following results hold:
    \begin{enumerate}
        \item The expected maximum of $a$, $\mathbb{E}[a]$ is $\Theta(\sqrt{\log n})$~\citep{kamath2015bounds}:
        \begin{align}
            \sqrt{\frac{\log n}{\pi \log 2}} \leq \E{a} \leq \sqrt{2 \log n}
        \end{align}
        \item Borell-TIS inequality: the maximum of Gaussian is well-concentrated \citep{adler1990introduction}:
        \begin{align}
            \prob{ \abs{a - \E{a}} \geq \sqrt{2 \log n} } \leq \frac{2}{n}
        \end{align}
        \item Consequently, $\abs{a} \leq \sqrt{8\log n}$ with probability at least $1 - \frac{2}{n}$.
    \end{enumerate}
\end{lemma}
The results of Lemma~\ref{lem: concentration of max of gaussians} naturally lead to several corollaries that will be utilized extensively throughout this work.

\begin{corollary}
    \label{cor: max of chi-squared random variables}
    Let $a_i \sim \normal(0, 1), i \in [n]$ be the $n$ i.i.d. Gaussian random variables. Define $b_i \coloneqq a_i^2$, and $b \coloneqq \max_{i \in [n]} b_i$. Then,
    \begin{align}
        \prob{ b \geq 8 \log n }\leq \frac{2}{n}~.
    \end{align}
\end{corollary}
\begin{proof}
    The result is a direct consequence of Lemma~\ref{lem: concentration of max of gaussians}.
    \begin{align}
        \prob{ b \geq 8 \log n } &= \prob{ \max_{i \in [n]} a_i^2 \geq 8 \log n } = \prob{ \big(\max_{i \in [n]} \abs{a_i} \big)^2 \geq 8 \log n }\\
        &= \prob{ \max_{i \in [n]} \abs{a_i}  \geq \sqrt{8 \log n} } \leq \frac{2}{n}.
    \end{align}
\end{proof}

\begin{corollary}
    \label{cor: max of <x_i, z>^p x_i1^q}
    Let $p, q \geq 0$ such that $p + q = 4$. Consider Gaussian random vectors $\bx_i \in \real^d, i \in [n]$ such that $x_{ij} \mathop{\sim}\limits_{iid} \normal(0, 1), \, \forall i \in [n],\, j \in [d]$, and a fixed $\bz \in \real^d$ such that $\| \bz \| = 1$. Then,
    \begin{align}
        \max_{i \in [n]} \abs{ \inner{\bx_i}{\bz}^p x_{ij}^q } \leq 64 \log^2 n
    \end{align}
    with probability at least $1 - \calO(\frac{1}{n})$.
\end{corollary}
\begin{proof}
    The result follows by noting that $\inner{\bx_i}{\bz} \sim \normal(0, 1)$ and $\abs{ \inner{\bx_i}{\bz}^p x_{ij}^q } = \abs{ \inner{\bx_i}{\bz}}^p \abs{ x_{ij} }^q$.
\end{proof}

\begin{corollary}
    \label{cor: conditional on max}
    Consider the following event defined using the notations from Corollary~\ref{cor: max of <x_i, z>^p x_i1^q}:
    \begin{align}
        \mathfrak{A} \coloneqq \Big\{ \max_{i \in [n]} \abs{\inner{\bx_i}{\bz}^p x_{ij}^q}  \leq 64 \log^2 n \Big\}~.
    \end{align}
    Then for any event $\mathfrak{B}$,
    \begin{align}
        \prob{\mathfrak{B}} \leq \prob{ \mathfrak{B} \given \mathfrak{A} }  + \calO\left(\frac{1}{n}\right)~.
    \end{align}
\end{corollary}
\begin{proof}
    Note that
    \begin{align}
        \prob{\mathfrak{B}} &= \prob{\mathfrak{B} \given \mathfrak{A}} \,  \prob{\mathfrak{A}} + \prob{\mathfrak{B} \given \neg \mathfrak{A} }\,  \prob{ \neg \mathfrak{A}} \\
        &\leq \prob{ \mathfrak{B} \given \mathfrak{A} }  + \calO\left(\frac{1}{n}\right)
    \end{align}
\end{proof}

\begin{lemma}
    \label{lem: hoeffding for <x_i, z>^p x_i1^q}
    Let $p, q \geq 0$ such that $p + q = 4$. Consider Gaussian random vectors $\bx_i \in \real^d, i \in [n]$ such that $x_{ij} \mathop{\sim}\limits_{iid} \normal(0, 1), \, \forall i \in [n],\, j \in [d]$. Then, $\forall \bz \in \real^d$ such that $\| \bz \| = 1$ and for any $t > 0$,
    \begin{align}
        \prob{ \abs{\frac{1}{n} \sum_{i=1}^n \inner{\bx_i}{\bz}^p x_{i1}^q - \E{ \frac{1}{n} \sum_{i=1}^n \inner{\bx_i}{\bz}^p x_{i1}^q } } \geq t } \leq 2 \exp\Big( - \frac{nt^2}{ C \log^4 n} + D d \Big) + \calO\Big(\frac{1}{n}\Big)~,
    \end{align}
    for a sufficiently large absolute constants $C, D > 0$.
\end{lemma}
\begin{proof}
    Consider a fixed $\bz \in \real^d$ such that $\| \bz \| = 1$. Let $\mathfrak{A}$ be the event defined in Corollary~\ref{cor: conditional on max}. Then using the result from Corollary~\ref{cor: conditional on max},
    \begin{align}
        \label{eq:cor 3 extension}
        \prob{ \abs{\frac{1}{n} \sum_{i=1}^n \inner{\bx_i}{\bz}^p x_{i1}^q - \E{ \frac{1}{n} \sum_{i=1}^n \inner{\bx_i}{\bz}^p x_{i1}^q } } \geq t } \leq \prob{\abs{\frac{1}{n} \sum_{i=1}^n \inner{\bx_i}{\bz}^p x_{i1}^q - \E{ \frac{1}{n} \sum_{i=1}^n \inner{\bx_i}{\bz}^p x_{i1}^q } } \geq t \given \mathfrak{A}} + \calO\left(\frac{1}{n}\right)~.
    \end{align}
    Using Hoeffding inequality~\citep{hoeffding1994probability} for the bounded-random variables:
    \begin{align}
    \label{eq: bounded hoeffding}
        \prob{\abs{\frac{1}{n} \sum_{i=1}^n \inner{\bx_i}{\bz}^p x_{i1}^q - \E{ \frac{1}{n} \sum_{i=1}^n \inner{\bx_i}{\bz}^p x_{i1}^q } } \geq t \given \mathfrak{A}} \leq 2 \exp\Big( - \frac{nt^2}{ C \log^4 n} \Big)~,
    \end{align}
    where $C > 0$ is an absolute constant. Equation~\eqref{eq: bounded hoeffding} holds for a fixed $\bz$. We can extend these to hold for any $ \bz \in \real^d$ such that $\| \bz \| = 1$ by using an $\varepsilon$-net argument and using a union bound across $\calO(2^d)$ points in the net. Therefore,  $\forall \bz \in \real^d$ such that $\| \bz \| = 1$ and for any $t > 0$:
    \begin{align}
    \label{eq: bounded hoeffding1}
        \prob{\abs{\frac{1}{n} \sum_{i=1}^n \inner{\bx_i}{\bz}^p x_{i1}^q - \E{ \frac{1}{n} \sum_{i=1}^n \inner{\bx_i}{\bz}^p x_{i1}^q } } \geq t \given \mathfrak{A}} \leq 2 \exp\Big( - \frac{nt^2}{ C \log^4 n}  + D d \Big)~,
    \end{align}
    where $D > 0$ is a sufficiently large constant. We complete the proof by combining the results from \eqref{eq:cor 3 extension} and \eqref{eq: bounded hoeffding1}.
\end{proof}

\begin{lemma}
\label{lem: bound on expectations}
    Let $p , q \geq 0$ such that $p + q = 4$ and $p \in \{2, 3, 4\}$. Consider Gaussian random vectors $\bx_i \in \real^d, i \in [n]$ such that $x_{ij} \mathop{\sim}\limits_{iid} \normal(0, 1), \, \forall i \in [n],\, j \in [d]$. Then, $\forall \bz \in \real^d$ such that $\| \bz \| = 1$,
    \begin{align}
        \E{ \frac{1}{n} \sum_{i=1}^n \inner{\bx_i}{\bz}^p x_{i1}^q } = C_{pq},
    \end{align}
    where $C_{pq} > 0$ is an absolute constant.
\end{lemma}
\begin{proof}
    The result follows from a straightforward verification using bounded moments of Gaussian random variables.
\end{proof}

Next part of our analysis establishes concentration results for sets of covariates $\bx_i$'s of sizes $(1 - \epsilon) n$ and $\epsilon n$. The core idea of our approach is inspired by the methodology outlined in the work of \citet{jambulapati2020robust}.

\begin{lemma}
    \label{lem: concentration for uncorrupted}
    Consider Gaussian random vectors $\bx_i \in \real^d$ for $i \in [n]$, where each $x_{ij} \mathop{\sim}\limits_{iid} \normal(0, 1)$ for all $i \in [n]$ and $j \in [d]$. Let $p, q \geq 0$ such that $p + q = 4$ with $p \in \{2, 3, 4\}$. For any unit vector $\bz \in \real^d$ (i.e., $| \bz | = 1$), and for any $0 < \epsilon < \frac{1}{2}$, $\delta > 0$, and subset $S \subseteq [n]$ with $\abs{S} = (1 - \epsilon) n$, provided that $n = \Omega\left( \frac{d + \log \frac{1}{\delta}}{\epsilon^2 \log \frac{1}{\epsilon}} \right)$, the following result holds:
    \begin{align}
        \label{eq: concentration for uncorrupted}
        \prob{ \abs{\frac{1}{(1 - \epsilon)n} \sum_{i \in S}  \inner{\bx_i}{\bz}^p x_{i1}^q - C_{pq} } \geq \epsilon \sqrt{\log \frac{1}{\epsilon}} \log^2(\epsilon n)  } \leq \calO\left(\delta\right) + \calO\left(\frac{1}{n}\right)~.
    \end{align}
\end{lemma}
\begin{proof}
    For any fixed $S \subseteq [n]$ such that $\abs{ S } = (1 - \epsilon) n$,
    \begin{align}
        \label{eq: decompose for uncorrupted}
        \frac{1}{(1 - \epsilon) n} \sum_{i\in S}  \inner{\bx_i}{\bz}^p x_{i1}^q = \frac{1}{1 - \epsilon} \Big( \frac{1}{n} \sum_{i=1}^n \inner{\bx_i}{\bz}^p \bx_{i1}^q \Big) -  \frac{\epsilon}{1 - \epsilon} \Big( \frac{1}{\epsilon n} \sum_{i \in [n] \setminus S} \inner{\bx_i}{\bz}^p x_{i1}^q \Big)
    \end{align}
Following the result from Lemma~\ref{lem: hoeffding for <x_i, z>^p x_i1^q},
\begin{align}
    \label{eq: concentration all}
    \prob{ \abs{ \frac{1}{n} \sum_{i=1}^n \inner{\bx_i}{\bz}^p x_{i1}^q - C_{pq} } \geq t  } \leq 2 \exp\left( \frac{-nt^2}{D_1 \log^4 n} + D_2 d \right) + \newO{\frac{1}{n}} ~,
    \end{align}
    for some absolute constants $D_1, D_2 > 0$. We take $t = \frac{1 - \epsilon}{2} \epsilon \sqrt{\log \frac{1}{\epsilon}} \log^2(\epsilon n)$ and $n = \Omega\big( \frac{d + \log \frac{1}{\delta}}{\epsilon^2 \log \frac{1}{\epsilon}} \big)$. This leads to,
    \begin{align}
        \label{eq: concentration all in terms of delta}
        \prob{ \abs{ \frac{1}{n} \sum_{i=1}^n \inner{\bx_i}{\bz}^p x_{i1}^q - C_{pq} } \geq  \frac{1 - \epsilon}{2} \epsilon \sqrt{\log \frac{1}{\epsilon}} \log^2(\epsilon n) } \leq \frac{\delta}{2} +  \newO{\frac{1}{n}}~,
    \end{align}
    Similarly, we can show that for some absolute constant $D_3, D_4 > 0$:
    \begin{align}
        \label{eq: concentration Sc}
        \prob{ \abs{ \frac{1}{\epsilon n} \sum_{i \in [n] \setminus S} \inner{\bx_i}{\bz}^p x_{i1}^q - C_{pq} } \geq t \given \mathfrak{A} } \leq 2 \exp\left( \frac{- \epsilon nt^2}{D_3 \log^4 (\epsilon n)} + D_4 d \right) ~,
    \end{align}
    where the event $\mathfrak{A}$ is defined in Corollary~\ref{cor: conditional on max}. We need \eqref{eq: concentration Sc} to hold across any choice of $S \subseteq [n]$. Thus, we take a union bound across $\binom{n}{\epsilon n}$ choices. Note that $\log \binom{n}{\epsilon n} \leq n \epsilon \log \frac{1}{\epsilon} $.  Therefore, for any choice of $S \subseteq [n]$,
    \begin{align}
        \label{eq: concentration Sc union}
       \prob{ \abs{ \frac{1}{\epsilon n} \sum_{i \in [n] \setminus S} \inner{\bx_i}{\bz}^p x_{i1}^q - C_{pq} } \geq t \given \mathfrak{A} } \leq 2 \exp\left( \frac{- \epsilon nt^2}{D_3 \log^4 (\epsilon n)} + D_4 d ) + n \epsilon \log \frac{1}{\epsilon} \right) ~,
    \end{align}
    We take $t = \frac{1 - \epsilon}{2}  \sqrt{\log \frac{1}{\epsilon}} \log^2(\epsilon n)$ and $n = \Omega\big( \frac{d + \log \frac{1}{\delta}}{\epsilon^2 \log \frac{1}{\epsilon}} \big)$ and this leads to
    \begin{align}
        \label{eq: concentration Sc union in terms of delta}
        \prob{ \abs{ \frac{1}{\epsilon n} \sum_{i \in [n] \setminus S} \inner{\bx_i}{\bz}^p x_{i1}^q - C_{pq} } \geq \frac{1 - \epsilon}{2}  \sqrt{\log \frac{1}{\epsilon}} \log^2(\epsilon n) \given \mathfrak{A} } \leq \frac{\delta}{2}~.
    \end{align}
    Following Corollary~\ref{cor: conditional on max}, and substituting the results of ~\eqref{eq: concentration all in terms of delta} and \eqref{eq: concentration Sc union in terms of delta} in ~\eqref{eq: decompose for uncorrupted}, we get
    \begin{align}
       \prob{ \abs{\frac{1}{(1 - \epsilon)n} \sum_{i \in S}  \inner{\bx_i}{\bz}^p x_{i1}^q - C_{pq} } \geq \epsilon \sqrt{\log \frac{1}{\epsilon}} \log^2(\epsilon n)  } \leq \calO\left(\delta\right) + \calO\left(\frac{1}{n}\right)~.
    \end{align}
\end{proof}

\begin{lemma}
\label{lem: concentration for corrupted}
    Adopting the notation from Lemma~\ref{lem: concentration for uncorrupted}, and for any choice of $0 < \epsilon < \frac{1}{2}$, $\delta > 0$, and $S \subseteq [n]$ such that $|S| = (1 - \epsilon) n$ and $n = \Omega\big( \frac{d + \log \frac{1}{\delta}}{\epsilon^2 \log \frac{1}{\epsilon}} \big)$, the following result holds:
    \begin{align}
        \label{eq: concentration for corrupted}
        \abs{ \frac{1}{n} \sum_{i \in [n]\setminus S}  \inner{\bx_i}{\bz}^p x_{i1}^q } \leq \newO{ \epsilon \sqrt{\log \frac{1}{\epsilon}} \log^2(\epsilon n) }
    \end{align}
    with probability at least $1 - \newO{\frac{1}{n}} - \newO{\delta}$.
\end{lemma}
\begin{proof}
    The result follows from the result of Lemma~\ref{lem: concentration for uncorrupted}.
    \begin{align}
         \frac{1}{n} \sum_{i \in [n]\setminus S}  \inner{\bx_i}{\bz}^p x_{i1}^q  &= \frac{1}{n} \sum_{i \in [n]}  \inner{\bx_i}{\bz}^p x_{i1}^q - \frac{1}{n} \sum_{i \in  S}  \inner{\bx_i}{\bz}^p x_{i1}^q \\
         &= \frac{1}{n} \sum_{i \in [n]}  \inner{\bx_i}{\bz}^p x_{i1}^q - (1- \epsilon)\frac{1}{(1 - \epsilon)n} \sum_{i \in  S}  \inner{\bx_i}{\bz}^p x_{i1}^q \\
         &= \frac{1}{n} \sum_{i \in [n]}  \inner{\bx_i}{\bz}^p x_{i1}^q - C_{pq} \notag \\
         &\quad - \Big( (1- \epsilon)\frac{1}{(1 - \epsilon)n} \sum_{i \in  S}  \inner{\bx_i}{\bz}^p x_{i1}^q - C_{pq} \Big) \\
         \abs{ \frac{1}{n} \sum_{i \in [n]\setminus S}  \inner{\bx_i}{\bz}^p x_{i1}^q  } &\leq \abs{ \frac{1}{n} \sum_{i \in [n]}  \inner{\bx_i}{\bz}^p x_{i1}^q - C_{pq} }  + (1- \epsilon) \abs{ \Big( \frac{1}{(1 - \epsilon)n} \sum_{i \in  S}  \inner{\bx_i}{\bz}^p x_{i1}^q - C_{pq} \Big) } + \epsilon \abs{ C_{pq} } \\
         &\leq \newO{ \epsilon \sqrt{\log \frac{1}{\epsilon}} \log^2(\epsilon n) }
    \end{align}
    with probability at least $1 - \newO{\frac{1}{n}} - \newO{\delta}$.
\end{proof}

\section{Proof of Proposition~\ref{lem:impossibility with constant corruption}}
\propimpossibility*
% \begin{proposition}[Impossibility with constant corruption proportion]
% If the measurements follow the data generation process~\eqref{eq: corrupted phase retrieval} with a corruption proportion $\epsilon > 0$, then for any estimator $\hat{\btheta}$ and any $\delta > 0$:
% \begin{align}
%     \mathbb{P} \big[ \| \hat{\btheta} - \btheta^* \| \geq \delta \big] \geq \frac{\epsilon}{2}~.
% \end{align}
% \end{proposition}
\begin{proof}
We show that proof in one dimension as an extension to $d$-dimension is straightforward. We consider the following phase retrieval model:
\begin{align}
y = (x \theta)^2 + \eta_{\theta, x}
\end{align}
where $\eta_{\theta, x}$ denotes the adversarial corruption added by a strong adversary who has access to both $x$ and $\theta$. We draw $x$ from a standard normal distribution. Consider two parameters $\theta_1 > 0$ and $\theta_2 > 0$ with $| \theta_1 - \theta_2 | > \delta$ for some $\delta > 0$.
% For any estimator $\hat{\theta}$, let
% \begin{align}
%     k = \arg\min_{j \in \{1, 2\}} |\hat{\theta} - \theta_j|~.
% \end{align}
% We choose $\theta^*$ to be $j$.

Let $D_1(x,y)$ and $D_2(x,y)$ be distributions over $\mathbb{R} \times \mathbb{R}$ corresponding to quadratic models $y = (x\theta_1)^2 + \eta_{\theta_1, x}$ and $y = (x\theta_2)^2 + \eta_{\theta_2, x}$ respectively. Since the adversary can only change $\epsilon$ fraction of inputs, we assume the following conditional distribution for $y$ conditioned on $x$ for $i \in \{1, 2\}$ for some $\sigma > 0$:
\begin{align}
D_i(y | x) = \begin{cases} 1 - \epsilon, \text{ when } y = (x\theta_i)^2 \\ \frac{\epsilon}{\sigma}, \text{ when } y \in [ \sigma,  2\sigma] \\ 0, \text{ otherwise} \end{cases}
\end{align}
We want to be able to differentiate between $D_1$ and $D_2$ based on the measurements $(x, y)$ drawn from either $D_1$ or $D_2$. By reduction to a hypothesis testing problem and using the Neyman--Pearson lemma:
\begin{align}
\inf_{\hat{\theta}}\sup_{\theta \in \{\theta_1, \theta_2\}}\mathbb{P}_{\theta}\big[|\hat{\theta} - \theta| > \delta \big] \geq \frac{1}{2}(1 - \mathrm{TV}(D_1, D_2))
\end{align}
where $\mathrm{TV}(D_1, D_2)$ is the total variation distance between distributions $D_1$ and $D_2$. Next, we compute an upper bound on $\mathrm{TV}(D_1, D_2)$.
\begin{align}
\mathrm{TV}(D_1, D_2) &= \frac{1}{2}\int_{\mathbb{R} \times \mathbb{R}} |D_1(x, y) - D_2(x, y)| \, \mathrm{d}x \, \mathrm{d}y \\
&= \frac{1}{2}\int_{\mathbb{R} \times \mathbb{R}} D_1(x) |D_1(y|x) - D_2(y|x)|\, \mathrm{d}x \, \mathrm{d}y
\end{align}
Notice that $D_1(y|x)$ and $D_2(y|x)$ can only differ when $(x\theta_1)^2 \ne (x\theta_2)^2$ and contribute $|D_1(y|x) - D_2(y|x)| \leq 2(1 - \epsilon)$ correspondingly. Overall,
\begin{align}
\mathrm{TV}(D_1, D_2) \leq 1 - \epsilon
\end{align}
It follows that,
\begin{align}
\inf_{\hat{\theta}}\sup_{\theta \in \{\theta_1, \theta_2\}}\mathbb{P}_{\theta}\big[|\hat{\theta} - \theta| > \delta \big] \geq \frac{\epsilon}{2}~.
\end{align}
\end{proof}

\section{Constructing \oracle{} and Proof of Theorem~\ref{thm: convergence of gradient descent}}
\label{sec:constructing oracle}

In this section, we analyze the behavior of the loss function in~\eqref{eq:nonconvex formulation} in the presence of corruption. Our discussion is framed within the context of Assumption~\ref{assum: independent corruption}, where we assume that the corruption $\eta_i$ in the response $y_i$ is independent of the covariates $\bx_i$. For this part of the analysis, by possibly reindexing the measurements, we define
\begin{align}
    \label{eq: f_U grad and hessian oracle 1}
    f_U(\btheta) &= \frac{1}{4 m}\sum_{i=1}^m \Big(\inner{\bx_i}{\btheta}^2 - \inner{\bx_i}{\btheta^*}^2 - \eta_i\Big)^2 \\
    \nabla f_U(\btheta) &= \frac{1}{m}\sum_{i=1}^m \Big(\inner{\bx_i}{\btheta}^2 - \inner{\bx_i}{\btheta^*}^2 - \eta_i \Big) \bx_i \bx_i^{\top} \btheta \\
    \nabla^2 f_U(\btheta) &= \frac{1}{m}\sum_{i=1}^m \Big(3 \inner{\bx_i}{\btheta}^2 - \inner{\bx_i}{\btheta^*}^2 - \eta_i \Big) \bx_i \bx_i^{\top}~.
\end{align}
Note that up to $k$ out of $m$ measurements may have corrupted responses $y_i$. Without loss of generality, we assume that $\btheta^* = [1, 0, \ldots, 0]^{\top}$. Taking the expectation over $\bx_1,\ldots,\bx_m$, we derive the following expected quantities:
\begin{align}
    \label{eq: population F_U grad and hessian oracle 1}
    F_U(\btheta) &= \frac{1}{4} \left(3  \| \btheta \|^4 + 3   - 4 \inner{\btheta}{\btheta^*}^2 - 2 \| \btheta \|^2   -  2 \|\btheta\|^2 \bar{\eta} + 2  \bar{\eta} + \frac{1}{m }\sum_{i=1}^m \eta_i^2 \right) \\
    \nabla F_U(\btheta) &= (3 \| \btheta \|^2 - 1) \btheta - 2 \inner{\btheta}{\btheta^*} \btheta^* - \bar{\eta} \btheta \\
    \nabla^2 F_U(\btheta) &= 6 \btheta \btheta^{\top} + 3 \| \btheta \|^2 - I - 2 \btheta^* {\btheta^*}^{\top} - \bar{\eta} I~,
\end{align}
where $\bar{\eta} = \frac{1}{m} \sum_{i=1}^m \eta_i$.

\subsection{Geometry of $F_U$}

\citet{sun2018geometric} investigated the geometry of the expected loss function in the absence of corruption. We extend this analysis to show that when the $\eta_i$'s are independent of the $\bx_i$'s, the geometry of $F_U(\btheta)$ also exhibits a benign structure. This is formalized by characterizing the critical points of $F_U(\btheta)$. At critical points,
\begin{align}
    \label{eq: critical points of F_U}
    \nabla F_U(\btheta) &= \bm{0} \\
    (3 \| \btheta \|^2 - 1) \btheta - 2 \inner{\btheta}{\btheta^*} \btheta^* - \bar{\eta} \btheta &= \bm{0}~.
\end{align}
We end up with three possible scenarios:
\begin{enumerate}
    \item $\btheta = \bm{0}$ is always a stationary point, but it behaves differently for different amount of average corruption.
    \begin{enumerate}
        \item When $\bar{\eta} \geq - 1$, $\bm{0}$ is the local maxima (technically, it can also be considered a strict saddle point).
        \item When $-3 < \bar{\eta} < -1 $, $\bm{0}$ is a strict saddle point.
        \item When $\bar{\eta} \leq -3$, $\bm{0}$ becomes the local (also global) minima due to the convexity of the $F_U(\btheta)$.
    \end{enumerate}
    \item When $\bar{\eta} \geq -1$, we can characterize a second set of critical points by a set
    \begin{align}
    \label{eq: strict saddle points}
    \mathcal{X} = \left\{ \btheta \given 3 \| \btheta \|^2 - 1 - \bar{\eta} = 0,\, {\btheta^*}^{\top} \btheta = 0 \right\}~.
    \end{align}
    They lead to strict saddle points.
    \item Finally, when $\bar{\eta} \geq -3$, we get another set of critical points.
    \begin{align}
    \label{eq: optimal points}
    \mathcal{X}^* = \left\{ \btheta \given \btheta = \pm \sqrt{1 + \frac{\bar{\eta}}{3}} \btheta^* \right\}~.
    \end{align}
    The points in $\mathcal{X}^*$ are the local (and global) minima.
\end{enumerate}

Notably, all critical points of $F_U(\btheta)$ are either strict saddle points or global minima. This suggests that the algorithms discussed in Section~\ref{sec: solving phase retrieval} are applicable for solving problem~\ref{eq:nonconvex formulation}, even in the presence of corrupted measurements. For our analysis, we employed gradient descent with random initialization, as proposed by \citet{chen2019gradient}.

\subsection{Gradient descent updates with $F_U$}

To gain intuition, we can study the dynamics of gradient descent with the (rather unrealistic) assumption that the gradient descent iterates $\tilde{\btheta}^t$ are independent of covariates $\bx_i, i \in [m]$. This leads to the following update rule:
\begin{align}
    \label{eq: population gd}
    \tilde{\btheta}^{t+1} = \tilde{\btheta}^t - \mu (3 \| \btheta \|^2 - 1) \btheta - 2 \inner{\btheta}{\btheta^*} \btheta^* - \bar{\eta} \btheta
\end{align}
where $\mu > 0$ is the chosen step size. We define the following two quantities:
\begin{align}
    \label{eq: alpha beta}
    \alpha_t =  \tilde{\theta}^t_1, \quad \beta_t = \sqrt{\sum_{i=2}^d (\tilde{\theta}^t_i)^2}
\end{align}
Without loss of generality, we can assume that $\alpha_0 > 0$. Equation \eqref{eq: population gd} leads to following dynamics for $\alpha_t$ and $\beta_t$:
\begin{align}
    \label{eq: population dynamics alpha beta}
    \alpha_{t+1}  &= \Big(1 + \mu \big( 3 + \bar{\eta} - 3 (\alpha_t^2 + \beta_t^2) \big) \Big) \alpha_t \\
    \beta_{t+1} &= \Big( 1 + \mu \big( 1 + \bar{\eta} - 3 (\alpha_t^2 + \beta_t^2)  \big) \Big) \beta_t
\end{align}
We observe that \eqref{eq: population dynamics alpha beta} has three fixed points $(\alpha, \beta)$:
\begin{enumerate}
    \item $(\alpha, \beta) = (0, 0)$ corresponds to $\btheta = \bm{0}$.
    \item When $\bar{\eta} \geq -1 $, then $(\alpha, \beta) = (0, \sqrt{\frac{1 + \bar{\eta}}{3}})$ corresponds to points in $\mathcal{X}$, defined in \eqref{eq: strict saddle points}.
    \item When $\bar{\eta} \geq -3 $, then $(\alpha, \beta) = (\sqrt{1 + \frac{\bar{\eta}}{3}}, 0)$ corresponds to points in $\mathcal{X}^*$, defined in \eqref{eq: optimal points}.
\end{enumerate}

In the absence of corruption, \citet{chen2019gradient} developed a ``leave-one-out'' technique to demonstrate that an approximately similar dynamic to \eqref{eq: population dynamics alpha beta} can be achieved using updates based on $f_U(\btheta)$, despite the gradient descent iterates $\tilde{\btheta}^t$ being dependent on the covariates $\bx_i$ for all $i \in [m]$. This framework is also applicable to our setting.

\subsection{Proof Sketch for Theorem~\ref{thm: convergence of gradient descent}}

In this subsection, we outline the key proof ideas for Theorem~\ref{thm: convergence of gradient descent}. The proof builds directly on the approach used in Theorem 2 of \citet{chen2019gradient}, allowing us to focus on the novel aspects that differentiate our work from theirs. Full details are omitted here to highlight the distinctions.

Consider the following dynamics for $\alpha_t$ and $\beta_t$ defined in \eqref{eq: alpha beta}:
\begin{align}
    \label{eq: sample dynamics alpha beta}
    \alpha_{t+1}  &= \Big(1 + \mu \big( 3 + \bar{\eta} - 3 (\alpha_t^2 + \beta_t^2) \big) + \mu \zeta_t \Big) \alpha_t~, \\
    \beta_{t+1} &= \Big( 1 + \mu \big( 1 + \bar{\eta} - 3 (\alpha_t^2 + \beta_t^2)  \big) + \mu \rho_t  \Big) \beta_t~,
\end{align}
where $\zeta_t$ and $\rho_t$ are the perturbation terms. Next, we discuss the major parts of the proof.

\subsubsection{When $f_U$ is convex} The first part deals with the case when $\bar{\eta} < -3$. In this scenario,  $f_U(\btheta)$ can be shown to be a convex function with high probability. To that end, we prove the following result.
\begin{lemma}
    \label{lem:convex loss landscape}
    Let $n = \Omega\big( \frac{d \, \mathrm{polylog}(d) + \log(\frac{1}{\delta})}{\epsilon^2 \log (\frac{1}{\epsilon})} \big)$ and $k \in \mathcal{K}$. If $\bar{\eta} = - 3  - \varepsilon $ for some $\varepsilon > 0$, then $f_U(\btheta)$  is a convex function with probability at least $1 - \delta - \newO{\frac{1}{n}}$.
\end{lemma}
\begin{proof}
    \label{proof: convex loss landscape}
    We proceed with studying the spectral properties of $\nabla^2 f_U(\btheta)$. We want to show that $\nabla^2 f_U(\btheta) \succeq 0$ with high probability. Recall that
    \begin{align}
        \nabla^2 f_U(\btheta) &= \frac{1}{m}\sum_{i=1}^m \Big(3 \inner{\bx_i}{\btheta}^2 - \inner{\bx_i}{\btheta^*}^2 - \eta_i \Big) \bx_i \bx_i^{\top} \notag
    \end{align}
    For any $\bz \in \real^d$ such that $\| \bz \| = 1$,
    \begin{align}
        \bz^{\top} \nabla^2 f_U(\btheta) \bz &= \bz^{\top} \nabla^2 f_U(\btheta) \bz - \bz^{\top} \nabla^2 F_U(\btheta)  \bz  + \bz^{\top} \nabla^2 F_U(\btheta) \bz~.
    \end{align}
    Observe that,
    \begin{align}
        \bz^{\top} \nabla^2 f_U(\btheta) \bz &= \frac{1}{m}\sum_{i=1}^m \Big(3 \inner{\bx_i}{\btheta}^2 \inner{\bx_i}{\bz}^2 - \inner{\bx_i}{\btheta^*}^2 \inner{\bx_i}{\bz}^2 - \eta_i \inner{\bx_i}{\bz}^2\Big)  \notag\\
        &= \frac{1}{m}\sum_{i=1}^m \Big(3 \inner{\bx_i}{\btheta}^2 \inner{\bx_i}{\bz}^2 - x_{i1}^2 \inner{\bx_i}{\bz}^2 - \eta_i \inner{\bx_i}{\bz}^2\Big)~.
    \end{align}
    and,
    \begin{align}
        \bz^{\top} \nabla^2 F_U(\btheta) \bz = 6 \inner{\btheta}{\bz}^2 + 3 \| \btheta \|^2 - 1 - 2 z_{1}^2 - \bar{\eta}~.
    \end{align}
    Using Lemma 14 from \citep{chen2019gradient}, if $n = \Omega(d\, \mathrm{polylog}(d))$, then for some absolute constant $c_0 > 0$, the following results hold with probability at least $1 - \calO(n^{-10})$:
    \begin{align}
        \frac{1}{m} \sum_{i=1}^m \Big(3 \inner{\bx_i}{\btheta}^2 \inner{\bx_i}{\bz}^2 - 6 \inner{\btheta}{\bz}^2 - 3 \| \btheta \|^2 \Big) &\geq - c_0 \sqrt{\frac{d \log^3 m}{m}} \| \btheta \|^2\\
        \frac{1}{m} \sum_{i=1}^m \Big( x_{i1}^2 \inner{\bx_i}{\bz}^2 -  1 - 2z_1^2 \Big) &\geq - c_0 \sqrt{\frac{d \log^3 m}{m}}
    \end{align}
    Next, notice that $\eta_i \inner{\bx_i}{\bz}$ is a subexponential random variable with parameters $(2 \eta_i, 4\eta_i)$.
    %It follows that $\| \eta_i \inner{\bx_i}{\bz_i} \|_{\psi_1} \leq 4\eta_i$.
    Using the Bernstein-type inequality~\citep{vershynin2010introduction}, we can write:
    \begin{align}
        \prob{ \abs{\frac{1}{m} \sum_{i=1}^m \big( \eta_i \inner{\bx_i}{\bz}^2 - \eta_i \big)} \geq  2 t \max_{i\in [m]} \abs{\eta_i} }  \leq 2\exp\bigg(-c m t^2\bigg)
    \end{align}
    for some constant $c > 0$ and $t \in [0, 1]$. By taking, $t = \sqrt{\frac{d \log m}{m}}$ and using a covering argument similar to \citet{chen2019gradient}, we can write $\forall \bz \in \real^d$ and $\| \bz \| = 1$,
    \begin{align}
        \frac{1}{m} \sum_{i \in [m]} \big( \eta_i \inner{\bx_i}{\bz}^2 - \eta_i \big) \geq - 2 \sqrt{\frac{d \log m}{m}} \max_{i\in U} \abs{\eta_i}
    \end{align}
    with probability at least $1 - \newO{\frac{1}{m}}$.
    Combining all the results, we have
    \begin{align}
        \bz^{\top} \nabla^2 f_U(\btheta) \bz &\geq -c_0 \sqrt{\frac{d \log^3 m}{m}} \| \btheta \|^2 -c_0 \sqrt{\frac{d \log^3 m}{m}} -  2 \sqrt{\frac{d \log m}{m}} \max_{i\in U} |\eta_i| + 6 \inner{\btheta}{\bz}^2 + 3 \| \btheta \|^2 - 1 - 2 z_{1}^2 - \bar{\eta} \\
        &\geq -c_0 \sqrt{\frac{d \log^3 m}{m}} \| \btheta \|^2 -c_0 \sqrt{\frac{d \log^3 m}{m}}  -  2 \sqrt{\frac{d \log m}{m}} \max_{i\in U} |\eta_i|  + 3 \| \btheta \|^2 + \varepsilon
    \end{align}

    By noticing that $\max_{i\in [m]} |\eta_i| = \calO(\log n)$, $m = (1 - 2\epsilon)n$ and taking $n = \Omega(\frac{d\, \mathrm{polylog}(d)}{\varepsilon^2})$, we show that $\forall \bz \in \real^d$ and with probability at least $1 - \delta - \newO{\frac{1}{n}}$
    \begin{align}
        \bz^{\top} \nabla^2 f_U(\btheta) \bz &\geq \frac{\varepsilon}{2}~.
    \end{align}
\end{proof}

Moreover, the global minimum of $f_U(\btheta)$ is attained at $\bm{0}$. Algorithm~\ref{alg: grad descent} leverages this property to return $\bm{0}$ when it estimates that the average corruption is less than $-3$. However, since the algorithm does not have direct access to the true value of $\bar{\eta}$, it requires a method to estimate the average corruption. For this purpose, we define the following quantity:
\begin{align}
    \label{eq: kappa sq}
    \kappa_{\mathrm{sq}} = \frac{1}{3|U|} \sum_{i \in U} \big( y_i z_i - (d - 1) y_i \big)~,
\end{align}
where $z_i = \sum_{j=1}^d x_{ij}^2, \forall i \in [m]$.

In the remaining part of the proof sketch, we assume that $\bar{\eta} \geq -3$. Next, we show that $\kappa_{\mathrm{sq}}$ provides a good estimation of $1 - \frac{\bar{\eta}}{3}$ with high probability.

We discuss the setting of $\kappa_{\mathrm{sq}}$ below:
\begin{align}
    \kappa_{\mathrm{sq}} = \frac{1}{3} \left( \sqrt{2} \sqrt{ \frac{1}{m}\sum_{i=1}^m y_i^2 - \left( \frac{1}{m} \sum_{i=1}^m y_i  \right)^2 }  + \frac{1}{m} \sum_{i=1}^m y_i\right)
\end{align}

Note that
\begin{align}
    \frac{1}{m}\sum_{i=1}^m y_i^2 = \frac{1}{m}\sum_{i=1}^m \left( \inner{\bx_i}{\btheta^*}^4 + \eta_i^2 + 2 \eta_i \inner{\bx_i}{\btheta^*}^2  \right)
\end{align}

Using the same argument as Lemma~\ref{lem: hoeffding for <x_i, z>^p x_i1^q} $\forall j \in [d]$,
\begin{align}
    \prob{\abs{\frac{1}{m} \sum_{i=1}^m \inner{\bx_i}{\btheta^*}^4 - \E{ \frac{1}{m} \sum_{i=1}^m \inner{\bx_i}{\btheta^*}^4} } \geq \varepsilon } \leq 2 \exp\left( - \frac{m\varepsilon^2}{ C \log^4 m} \right) + \newO{\frac{1}{m}}~,
\end{align}
Using the Bernstein-type inequality~\citep{vershynin2010introduction} for subexponential random variables $\eta_i \inner{\bx_i}{\btheta^*}^2, \forall i \in [m]$:
\begin{align}
    \prob{ \abs{\frac{1}{m} \sum_{i=1}^m \big( \eta_i \inner{\bx_i}{\btheta^*}^2 - \eta_i \big)} \geq  \varepsilon \max_{i \in [m]} \abs{\eta_i} }  \leq 2 \exp\left(-c m \varepsilon^2 \right)~,
\end{align}
for $\varepsilon \in (0, 1)$ and some absolute constant $c > 0$. Similarly, $\frac{1}{m} \sum_{i=1}^m y_i$ concentrates sharply around $\| \btheta^* \|^2 + \bar{\eta}$. Combining the above results together, we can show that
\begin{align}
    \left( 1 + \frac{\bar{\eta}}{3} \right) - \newO{\varepsilon} \leq \kappa_{\mathrm{sq}} \leq  \left( 1 + \frac{\bar{\eta}}{3} \right) + \newO{\varepsilon} + \newO{\epsilon \log^2 m}~,
\end{align}
with probability at least $1 - \newO{\delta} - \newO{\frac{1}{m}}$.

\subsubsection{When approximate dynamics for $\alpha_t$ and $\beta_t$ holds}

Following a similar line of reasoning as in \citet{chen2019gradient}, the subsequent part of the proof demonstrates that if the dynamics in \eqref{eq: sample dynamics alpha beta} hold for $\zeta_t = \newO{\frac{1}{\log d}}$ and $\rho_t = \newO{\frac{1}{\log d}}$, then there exists some $\nu \in (0, 1)$ and a corresponding $T_0 = T_0(\nu) = \newO{\log d}$ such that:
\begin{align}
    \abs{\alpha_{T_0} - \kappa} \leq \frac{\nu}{2}, \quad  \beta_{T_0} \leq \frac{\nu}{2}~.
\end{align}
This result implies that $d(\tilde{\btheta}^{T_0}, \kappa \btheta^*) \leq \nu$. Achieving this bound relies on an effective initialization, which is attained by setting $\tilde{\btheta}^0 = \sqrt{\kappa_{\mathrm{sq}}}\bm{u}$, where $\bm{u}$ is uniformly distributed on the unit sphere. The arguments closely follow the reasoning presented in the proof of Theorem 3 in \citet{chen2019gradient}.

\subsubsection{Justification for approximate dynamics of $\alpha_t$ and $\beta_t$}

\citet{chen2019gradient} employ a variant of leave-one-out arguments to demonstrate that the dynamics described in \eqref{eq: sample dynamics alpha beta} hold, with $\zeta_t = \newO{\frac{1}{\log d}}$ and $\rho_t = \newO{\frac{1}{\log d}}$. Their approach is based on constructing three specific leave-one-out sequences: the $l$-th leave-one-out sequence, the random sign sequence, and the $l$-th leave-one-out with random sign sequence. These sequences are instrumental in establishing a form of \emph{near-independence} between the iterates $\tilde{\btheta}^t$ and the covariates $\bx_i$ for all $i \in [m]$. Below, we provide formal definitions of these four sequences of iterates (including the original sequence) and outline their respective update rules:
\begin{tabular}{ll}
     Original sequence: &  $\nabla f(\btheta) = \frac{1}{m}\sum_{i=1}^m \Big(\inner{\bx_i}{\btheta}^2 - \inner{\bx_i}{\btheta^*}^2 - \eta_i \Big) \bx_i \bx_i^{\top} \btheta$  \\
     & $\tilde{\btheta}^{t+1} = \tilde{\btheta}^t - \mu \nabla f(\tilde{\btheta}^t)$\\
     $l$-th leave-one-out sequence: & $\nabla f^{(l)}(\btheta) = \frac{1}{m}\sum_{i=1, i \ne l}^m \Big(\inner{\bx_i}{\btheta}^2 - \inner{\bx_i}{\btheta^*}^2 - \eta_i \Big)  \bx_i \bx_i^{\top} \btheta  $ \\
     & $\tilde{\btheta}^{t+1, (l)} = \tilde{\btheta}^{t, (l)} - \mu \nabla f^{(l)}(\tilde{\btheta}^{t,(l)})$ \\
     Random sign sequence: & $\nabla f^{\mathrm{sgn}}(\btheta) = \frac{1}{m}\sum_{i=1}^m \Big(\inner{\bx_i^{\mathrm{sgn}}}{\btheta}^2 - \inner{\bx_i^{\mathrm{sgn}}}{\btheta^*}^2 - \eta_i \Big)  \bx_i^{\mathrm{sgn}} {\bx_i^{\mathrm{sgn}}}^{\top} \btheta$ \\
     & $\tilde{\btheta}^{t+1, \mathrm{sgn}} = \tilde{\btheta}^{t, \mathrm{sgn}} - \mu \nabla f^{\mathrm{sgn}}(\tilde{\btheta}^{t,\mathrm{sgn}})$ \\
     $l$-th leave-one-out and random sign sequence: & $\nabla f^{(l), \mathrm{sgn}}(\btheta) = \frac{1}{m}\sum_{i=1, i \ne l}^m \Big(\inner{\bx_i^{ \mathrm{sgn}}}{\btheta}^2 - \inner{\bx_i^{ \mathrm{sgn}}}{\btheta^*}^2 - \eta_i \Big) \bx_i^{\mathrm{sgn}} {\bx_i^{\mathrm{sgn}}}^{\top} \btheta   $ \\
     &  $\tilde{\btheta}^{t+1, \mathrm{sgn}, (l)} = \tilde{\btheta}^{t , \mathrm{sgn}, (l)} - \mu \nabla f^{(l), \mathrm{sgn}}(\tilde{\btheta}^{t, \mathrm{sgn},(l)})$
\end{tabular}
The notations used here follow closely from \citet{chen2019gradient}. Specifically, for a given $\bx_i^{\mathrm{sgn}}$, we define $x_{ij}^{\mathrm{sgn}} = x_{ij}$ for $j \neq 1$ and $x_{i1}^{\mathrm{sgn}} = w_i x_{i1}$, where $w_i$ is a Rademacher random variable. All sequences are initialized with $\tilde{\btheta}^0$, and a constant step size $\mu > 0$ is employed.

It is important to note that the analysis involves additional concentration inequalities due to the presence of corruption. Specifically, we need to ensure that the absence of terms such as $\mu \frac{1}{m} \bx_l \bx_l^{\top} \btheta^{t, (l)}$ does not cause the gradient $\nabla f^{(l)}(\btheta)$ to deviate significantly from $\nabla f(\btheta)$. We demonstrate that such deviations remain controlled, ensuring the robustness of the overall analysis.
\begin{align}
    \| \mu \frac{1}{m} \eta_i \bx_l \bx_l^{\top} \btheta^{t, (l)} \| &\mathop{\lesssim}\limits_{(i)} \mu \frac{1}{m} \log m \abs{\bx_l^{\top} \btheta^{t, (l)}} \| \bx_l \| \\
    &\mathop{\lesssim}\limits_{(ii)} \mu \frac{1}{m} \log m \sqrt{\log m} \| \btheta^{t, (l)}  \| \sqrt{d} \\
    &\lesssim \mu \frac{\sqrt{d\log^3 m}}{m} \| \btheta^{t, (l)}  \|~.
\end{align}
In the above derivation, step $(i)$ follows from the fact that $\eta_i = \newO{\log m}$, while step $(ii)$ leverages the tail bounds for the maximum of a standard Gaussian variable and the norm of a Gaussian vector. A similar reasoning applies to other corruption-related terms that emerge in the analysis. The rest of the analysis is similar to \citet{chen2019gradient}.

\section{Experimental Comparisons}

We evaluated the performance of our method against established approaches for robust phase retrieval, specifically comparing with Median RWF~\citep{zhang2016provable} and PhaseLift~\citep{hand2017phaselift}. Median RWF employs spectral initialization, while PhaseLift relies on a convex SDP formulation. The covariates $\bx_i$ were sampled from a standard normal distribution, and the corruption was uniformly distributed within the range $[-5, 5]$. We conducted experiments with $k = n^{\frac{2}{3}}$ and $n = 10 d \log d$ for $d \in \{50, 500, 1000\}$. Algorithm~\ref{alg: grad descent} served as \oracle{} in our method. All the first-order updates (inner loop in our method and gradient descent type updates in Median RWF) were run for $500$ iterations. Performance metrics included relative error, defined as $\frac{d(\btheta, \btheta^*)}{\| \btheta^* \|}$, and runtime, with results averaged over 5 independent runs. All methods were implemented in MATLAB and tested on a MacBook Pro with macOS 14.4.1, 32 GB memory, and an Apple M2 Max chip. For PhaseLift, CVX was employed as the SDP solver, and experiments were manually terminated if no solution was found within 5 minutes.

\begin{table}[!ht]
\centering
\caption{Comparison of the performance of different methods across various values of $d$. Runtime is measured in seconds, with all values rounded to three decimal places.}
\begin{small}
\begin{tabular}{*7c}
\toprule
Method &  \multicolumn{2}{c}{$d = 50$} & \multicolumn{2}{c}{$d=500$} & \multicolumn{2}{c}{$d=1000$}\\
\midrule
{}   & Rel Error   & Run time   & Rel Error   & Run time & Rel Error   & Run time\\
\altmin{}   &  $0.003 \pm 0.002$ & $0.553 \pm 0.01$   & $0.000 \pm 0.000$  & $13.060 \pm 0.303$ & $0.000 \pm 0.000$ & $56.519 \pm 0.690$\\
Median RWF   &  $0.000 \pm 0.000$ & $0.404 \pm 0.015 $   & $0.000 \pm 0.000$  & $ 24.907 \pm 0.518 $ & $ 0.000 \pm 0.000 $ & $ 139.219 \pm 0.926 $\\
PhaseLift   &  $0.000 \pm 0.000 $  &  $102.093 \pm 6.945 $   & NA  & $> 300$& NA & $> 300$\\
\bottomrule
\end{tabular}
\end{small}
\end{table}

For $d = 50$, all methods demonstrated comparable performance, though PhaseLift exhibited the highest runtime due to its SDP-based approach. As $d$ increased, PhaseLift failed to produce results within the 5-minute threshold. For larger dimensions, $d \in \{500, 1000\}$, our method showed performance comparable to Median RWF in terms of relative error, but it significantly outperformed Median RWF in terms of runtime efficiency.

% \begin{lemma}
%     \label{lem: finite sample opt}
%     Under Assumption~\ref{assum: independent corruption} and for any $\rho > 0$, if $n = \Omega\bigg(\frac{ \big(d + \log \frac{1}{\delta} \big) \, \mathrm{polylog}(d) }{\rho^2}\bigg)$, then the following results hold with probability at least $1 - \delta - \calO(\frac{1}{n})$.
%     \begin{enumerate}
%         \item If $\bar{\eta} \leq -3 \| \btheta^* \|^2$, then $ f_U(\bm{0}) \leq \min_{\btheta \in \real^d} f_U(\btheta) + \rho~ $, and
%         \item if $\bar{\eta} > -3 \| \btheta^* \|^2$, then $ f_U(\kappa \btheta^*) \leq \min_{\btheta \in \real^d} f_U(\btheta) + \rho$.
%     \end{enumerate}
% \end{lemma}

\end{document}